\documentclass{article}

\usepackage{microtype}
\usepackage{graphicx}
\usepackage{subfigure}
\usepackage{subcaption}
\usepackage{booktabs} % for professional tables

\usepackage[utf8]{inputenc}
\usepackage{cuted}
\usepackage{tabularx}

\usepackage{hyperref}

\usepackage[preprint]{icml2026}

\usepackage{amsmath}
\usepackage{amssymb}
\usepackage{mathtools}
\usepackage{amsthm}
\usepackage[linesnumbered,ruled,vlined,algo2e]{algorithm2e}
\DeclareMathOperator*{\argmax}{arg\,max}

\usepackage[capitalize,noabbrev]{cleveref}

\theoremstyle{plain}
\newtheorem{theorem}{Theorem}[section]
\newtheorem{proposition}[theorem]{Proposition}
\newtheorem{lemma}[theorem]{Lemma}
\newtheorem{corollary}[theorem]{Corollary}
\theoremstyle{definition}
\newtheorem{definition}[theorem]{Definition}
\newtheorem{assumption}[theorem]{Assumption}
\theoremstyle{remark}
\newtheorem{remark}[theorem]{Remark}

\usepackage{multirow}
\usepackage{xspace}
\usepackage{booktabs}
\usepackage{soul}
\usepackage[capitalize,noabbrev]{cleveref}
\usepackage{url}
\usepackage{breakurl}
\usepackage{dblfloatfix}

\usepackage{wrapfig, lipsum}
\usepackage{bussproofs}
\usepackage{scalerel}
\usepackage{graphicx}
\newcommand{\nback}[1][-.95pt]{
  \mathrel{\raisebox{#1}{$\rotatebox[origin=c]{-315}{\scaleobj{0.55}{-}}$}}
}
\newcommand{\undernegpreccurlyeq}{%
\mathrel{\ooalign{$\preccurlyeq$\cr\kern1.2pt$\nback$}}}
\usepackage{subcaption}

\usepackage[textsize=tiny]{todonotes}

\icmltitlerunning{Beyond Message Passing: A Symbolic Alternative for Expressive and Interpretable Graph Learning}

\begin{document}

\twocolumn[
\icmltitle{Beyond Message Passing: A Symbolic Alternative for Expressive and Interpretable Graph Learning}

  % It is OKAY to include author information, even for blind submissions: the
  % style file will automatically remove it for you unless you've provided
  % the [accepted] option to the icml2026 package.

  % List of affiliations: The first argument should be a (short) identifier you
  % will use later to specify author affiliations Academic affiliations
  % should list Department, University, City, Region, Country Industry
  % affiliations should list Company, City, Region, Country

  % You can specify symbols, otherwise they are numbered in order. Ideally, you
  % should not use this facility. Affiliations will be numbered in order of
  % appearance and this is the preferred way.
  \icmlsetsymbol{equal}{*}

  \begin{icmlauthorlist}
    \icmlauthor{Chuqin Geng}{McGill,UofT}
    \icmlauthor{Li Zhang}{UofT}
    \icmlauthor{Haolin Ye}{McGill}
    \icmlauthor{Ziyu Zhao}{McGill}
    \icmlauthor{Yuhe Jiang}{UofT}
    \icmlauthor{Tara Saba}{UofT}
    \icmlauthor{Xinyu Wang}{McGill}
    \icmlauthor{Xujie Si}{UofT}
    
  \end{icmlauthorlist}

  \icmlaffiliation{McGill}{
    School of Computer Science, McGill University, Montreal, Canada
  }

  \icmlaffiliation{UofT}{
    Department of Computer Science, University of Toronto, Toronto, Canada
  }

  \icmlcorrespondingauthor{Chuqin Geng}{chuqin.geng@mail.mcgill.ca}
  \icmlcorrespondingauthor{Xujie Si}
  {six@cs.toronto.edu}
  % You may provide any keywords that you find helpful for describing your
  % paper; these are used to populate the "keywords" metadata in the PDF but
  % will not be shown in the document
  \icmlkeywords{Machine Learning, ICML}

  \vskip 0.3in
]

% this must go after the closing bracket ] following \twocolumn[ ...

% This command actually creates the footnote in the first column listing the
% affiliations and the copyright notice. The command takes one argument, which
% is text to display at the start of the footnote. The \icmlEqualContribution
% command is standard text for equal contribution. Remove it (just {}) if you
% do not need this facility.

% Use ONE of the following lines. DO NOT remove the command.
% If you have no special notice, KEEP empty braces:
% \printAffiliationsAndNotice{}  % no special notice (required even if empty)
% Or, if applicable, use the standard equal contribution text:
\printAffiliationsAndNotice{}

\begin{abstract}
Graph Neural Networks (GNNs) have become essential in high-stakes domains such as drug discovery, yet their black-box nature remains a significant barrier to trustworthiness. While self-explainable GNNs attempt to bridge this gap, they often rely on standard message-passing backbones that inherit fundamental limitations, including the 1-Weisfeiler-Lehman (1-WL) expressivity barrier and a lack of fine-grained interpretability. To address these challenges, we propose \textsc{SymGraph}, a symbolic framework designed to transcend these constraints. By replacing continuous message passing with discrete structural hashing and topological role-based aggregation, our architecture theoretically surpasses the 1-WL barrier, achieving superior expressiveness without the overhead of differentiable optimization. Extensive empirical evaluations demonstrate that \textsc{SymGraph} achieves state-of-the-art performance, outperforming existing self-explainable GNNs. Notably, \textsc{SymGraph} delivers $10\times$ to $100\times$ speedups in training time using only CPU execution. Furthermore, \textsc{SymGraph} generates rules with superior semantic granularity compared to existing rule-based methods, offering great potential for scientific discovery and explainable AI.
\end{abstract}

\section{Introduction}
\label{sec:intro}
Graph Neural Networks (GNNs) have emerged as the de facto standard for modeling graph-structured data, achieving remarkable performance in domains from drug discovery \cite{xiong2021,liu2022} to fraud detection \cite{rao2021}. Despite their success, GNNs encode structural patterns into numerical representations, making it challenging to understand the rationale behind their predictions.

% To address this opacity, the field of Explainable AI (XAI) is witnessing a fundamental shift in two dimensions: a move from \emph{post-hoc} to \emph{self-explainable} architectures, and a transition from \emph{attribution} to \emph{logical reasoning}. While early methods focused on post-hoc explainers \cite{GNNExplainer,PGExplainer} or identifying influential motifs \cite{pope2019,xgnn}, these approaches face a critical semantic gap: they merely isolate salient substructures without explicating the compositional rules that determine their importance. Consequently, the frontier of XAI has pivoted toward \emph{logic-based self-explainable models} \cite{DBLP:journals/inffus/SchnakeJLXNGMM25}, which aim to formulate explicit logical rules that transparently reveal \textit{how} the decision is made.

To address this opacity, the field of Explainable AI (XAI) is witnessing a fundamental shift in two dimensions: a move from \emph{post-hoc} to \emph{self-explainable} architectures, and a transition from \emph{attribution} to \emph{logical reasoning}. While early methods focused on post-hoc explainers to highlight salient graph elements \cite{pope2019,GNNExplainer} or discover influential motifs \cite{xgnn,PGExplainer}, these approaches face dual limitations. First, being post-hoc approximations, they inherently lack faithfulness, often failing to accurately reflect the underlying decision process of the black-box model. Second, they face a critical semantic gap: they merely isolate salient substructures without explicating the compositional rules that determine their importance. Consequently, the frontier of XAI has pivoted toward \emph{logic-based, self-explainable models} \cite{DBLP:journals/inffus/SchnakeJLXNGMM25}, which aim to formulate explicit logical rules that transparently reveal \textit{how} the decision is made.

A prominent line of research, exemplified by logical frameworks like \textsc{LogiX-GIN} \cite{LogiXGIN}, seeks to achieve intrinsic interpretability by imposing structural constraints directly into the GNN backbone. However, despite these advancements, we identify two critical limitations that restrict both the theoretical power and practical utility of the entire class of self-explainable GNNs (SE-GNNs) \cite{pignn,gib,kergnn}. First, by relying on continuous GNN backbones, these approaches suffer from a structural bottleneck. They inevitably inherit the theoretical ceiling of the Weisfeiler-Lehman (1-WL) test \cite{GIN}, capping their expressive power, while simultaneously incurring significant computational overhead due to costly end-to-end differentiable optimizations or auxiliary distillation steps. Second, existing explanations, whether based on latent prototypes/subgraphs \cite{pignn} or logic rules over such subgraphs, often lack the semantic granularity required for practical applications. For example, chemistry experts reason through precise structural constraints regarding atom types and bonding configurations, such as SMARTS patterns~\cite{daylight2019smarts}. Such strict topological constraints are rarely captured by current explanation paradigms.

% \textsc{SymGraph} resolves this by grounding explanations in logical predicates that mirror these expert-defined rules, rendering the entire decision mechanism explicit~\cite{LogicXGNN}.

% Second, learned explanations, whether defined by latent prototypes or logic rules over them, are often coarse-grained representations that lack \emph{semantic granularity}, failing to provide the precise, node/edge topological constraints necessary to validate findings against expert domain knowledge. For instance,  experts in chemsary reason in terms of precise rules and structural constraints, uch as SMARTS-style specifications~\cite{daylight2019smarts} regarding allowable atom types and bonding configurations. strict constraints cannot be revealed by individual subgraph explanations alone. Our framework resolves this by grounding elements in logical predicates that mirror these expert-defined rules, making the entire decision mechanism explicit~\cite{LogicXGNN}.

We contend that overcoming these barriers requires moving beyond \emph{neural message passing paradigms} toward developing standalone symbolic engines where logic serves as the model itself. To this end, we propose \textsc{SymGraph}, the first standalone symbolic framework for graph learning that is designed to transcend the bottlenecks of neural backbones. Specifically, \textsc{SymGraph} leverages discrete structural hashing and topological role-based aggregation to construct logical programs theoretically proven to be strictly more powerful than the 1-WL limit. These theoretical claims are supported by our empirical evaluations. By replacing expensive gradient descent with efficient combinatorial evolutionary search, our framework enables orders-of-magnitude faster training on CPU only while producing structural rules with higher granularity and precision than existing SE-GNNs. These rules offer superior utility for XAI, supporting scientific discovery by aligning directly with expert domain knowledge. We summarize our contributions as follows:

\begin{itemize}
    \item We propose \textsc{SymGraph}, the first standalone symbolic framework theoretically proven to break the 1-WL expressivity barrier by leveraging discrete structural hashing and topological role-based aggregation.

    \item We employ an efficient combinatorial evolutionary search over candidate symbolic programs, overcoming the traditional overhead of differentiable optimization.

    \item  Extensive empirical evaluations demonstrate that \textsc{SymGraph} consistently outperforms state-of-the-art self-explainable models while achieving $10\times$ to $100\times$ speedups in training time using only CPU execution.

    \item Thanks to its symbolic nature, \textsc{SymGraph} enforces strict structural constraints on nodes and edges, achieving finer semantic granularity than existing methods. Notably, \textsc{SymGraph} outputs align directly with SMARTS patterns, the standard 
   formalism in cheminformatics, enabling automated recovery of expert-validated 
   Structure-Activity Relationships, a capability unattainable with existing 
   coarse-grained explanations, as demonstrated on Mutagenicity and BBBP.

\end{itemize}

\section{Preliminary}
\label{sec:back}
\subsection{Graph Neural Networks}
We consider a graph \( G = (V_G, E_G) \), where \( V_G \) denotes the set of nodes and \( E_G \) the set of edges. Let \( \mathbf{X} \in \mathbb{R}^{n \times d_0} \) represent the node feature matrix and \( \mathbf{A} \in \{0, 1\}^{n \times n} \) the adjacency matrix. Standard Graph Neural Networks, specifically within the Message Passing Neural Network (MPNN) framework, generally operate in two distinct phases: (1) \emph{Message Passing} to encode local structural information into continuous node embeddings, and (2) \emph{Global Aggregation} to fuse these encodings into task-specific predictions.

\paragraph{Phase 1: Message Passing (Local Encoding).}
The primary objective of this phase is to learn a node embedding \( \mathbf{x}_v^{(L)} \in \mathbb{R}^{d_L} \) that encodes the local context of node \( v \), integrating both its topological structure and semantic feature information. Through an iterative mechanism spanning \( L \) layers, the embedding is updated as follows:
\begin{equation}
\mathbf{x}_v^{l+1} = \text{UPD} \left( \mathbf{x}_v^{l}, \text{AGG}_{local} \left( \left\{ \text{MSG} (\mathbf{x}_v^{l}, \mathbf{x}_w^{l}) \mid \mathbf{A}_{vw} = 1 \right\} \right) \right)
\end{equation}
where \( \mathbf{x}_v^{0} = \mathbf{X}_v \) denotes the initial features. The functions \( \text{MSG} \), \( \text{AGG}_{local} \), and \( \text{UPD} \) define the specific GNN architecture, with prominent examples including GCN \cite{gcn} and GIN \cite{GIN}. The final embedding \( \mathbf{x}_v^{L} \) thus serves as a compact representation of the node's role within its $L$-hop  receptive field (induced subgraph), which we denote as $S_v$.

\paragraph{Phase 2: Global Aggregation.}
Once the local embeddings are generated, the second phase aggregates them depending on the downstream task.

\textit{Graph Classification.}
For a graph-level task (mapping \( G \) to a label \( y_G \)), a global aggregation (readout) function is required to summarize the entire graph into a single representation. This constitutes the aggregation stage:
\begin{equation}
    \mathbf{x}_G^{L} = \text{AGG}_{global}(\{ \mathbf{x}_v^{L} \mid v \in V_G \}),
\end{equation}
where \( \text{AGG}_{global} \) is typically a permutation-invariant operation, just like \( \text{AGG}_{local} \). A common choice is \texttt{global\_mean\_pooling}, defined as \( \mathbf{x}_G^{L} := \text{mean}(\mathbf{x}_v^{L} \mid v \in V_G) \). The graph-level embedding \( \mathbf{x}_G^{L} \) is then passed to a final MLP layer to predict the graph label.

\textit{Node Classification.}
For a node-level task (mapping \( v \) to a label \( y_v \)), this aggregation step is omitted. Instead, the prediction is derived directly from the local embedding \( \mathbf{x}_v^{L} \).

\subsection{Logical Rule-Based Models for Graph Learning}
\label{sec:logic_prelim}
The theoretical foundation for mapping GNNs to formal logic is provided by \citet{DBLP:conf/iclr/BarceloKM0RS20}, who identify Graded Modal Logic (GML) \cite{DBLP:journals/sLogica/Rijke00} as the corresponding formalism for standard message passing. Building on these insights, we design our framework as a standalone symbolic model that mirrors the two-phase architecture of GNNs using purely logical components:

\paragraph{Level 1: Local Structural Predicates.}
The foundation of the model consists of a set of atomic predicates $\mathcal{P} = \{p_1, \dots, p_J\}$, which serve as the symbolic counterpart to message passing. Unlike standard GNN layers that aggregate neighborhood features indiscriminately into latent vectors, our predicates enforce strict feature-level constraints tied to structural roles (orbits) within the neighborhood.

Formally, let $\mathcal{O}_v = \{O_1, \dots, O_B\}$ denote the partition of node $v$'s local neighborhood $S_v$ into $B$ distinct topological orbits (grouping structurally equivalent nodes or edges). A predicate $p_j(v)$ evaluates to True if the aggregated features within these orbits satisfy a logical function:
\begin{equation}
    p_j(v) \iff \Psi_j\big( \psi(O_1, \mathbf{X}), \dots, \psi(O_B, \mathbf{X}) \big)
\end{equation}
Here, $\psi(O_i, \mathbf{X})$ represents a set of validity evaluations derived from comparing the aggregated features of orbit $O_i$ against learned thresholds (e.g., evaluating whether ``count of Carbon $> 1$'' is True), as shown in Figure~\ref{fig:orbit_rule_eg}. Consequently, $\Psi_j$ acts as a structured logical filter (e.g., in Disjunctive Normal Form) that determines whether these local signals constitute a valid structural pattern.

\begin{figure}[t]
    \centering
    \includegraphics[width=\linewidth]{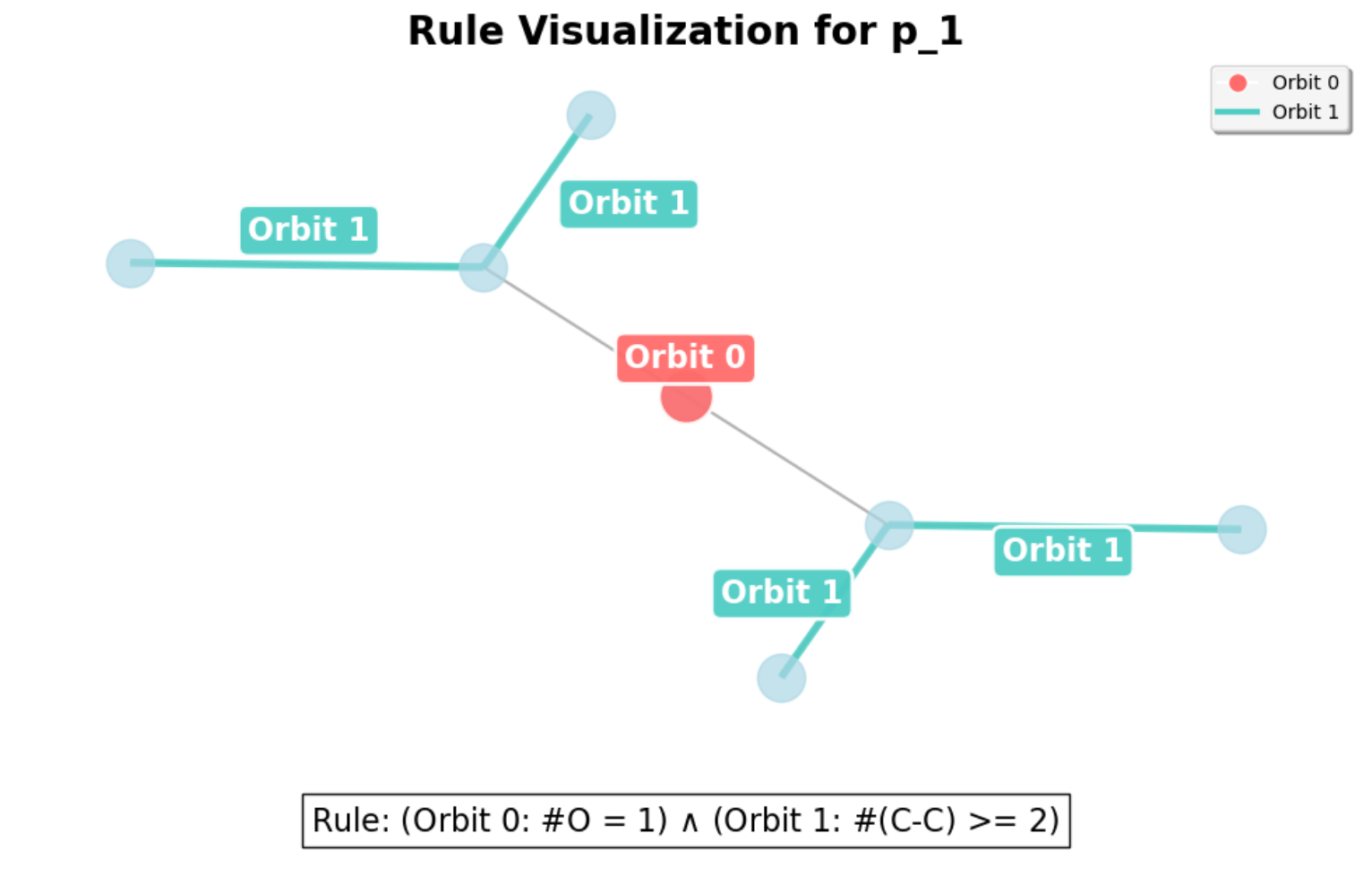}
    \caption{Visualization of a Local Structural Predicate. The extracted predicate identifies specific node and edge constraints, providing superior semantic granularity compared to canonical subgraph-based explanations. See more details in Appendix \ref{app:examples_read}.}
    \label{fig:orbit_rule_eg}
    \vspace{-5pt}
\end{figure}

\paragraph{Level 2: Global Logic Formula.} The second level translates the symbolic signals from Level 1 into a final decision.

\textit{Graph Classification.} For graph-level tasks, we perform a symbolic readout that aggregates local detections into a class-wise decision. We adopt a \textit{Quantitative Disjunctive Normal Form (Q-DNF)}, which expresses the global decision not as a black-box sum or mean, but as a structured logical rule defined over predicate frequencies. Let $C_j(G) = \sum_{v \in V_G} \mathbb{I}(p_j(v))$ be the total count of predicate $p_j$ in graph $G$. The model assigns a label $y$ if the graph satisfies a global formula $\Phi$:
\begin{equation}
    \Phi(G) = \bigvee_{m=1}^{M} \left( \bigwedge_{(j, \text{op}, \kappa) \in \mathcal{R}_m} (C_j(G) \text{ op } \kappa) \right) \implies y
\end{equation}
where $M$ represents the total number of distinct logical rules, and each rule $\mathcal{R}_m$ is a conjunction of constraints defined by an integer threshold $\kappa \in \mathbb{N}$ (e.g., $C_j \ge 2$) and an operator $\text{op} \in \{\ge, <, =\}$. By reasoning about pattern \textit{multiplicity} rather than mere existence, this formulation replaces opaque global pooling with transparent, quantitative logic.

\textit{Node Classification.} For node-level tasks,  the decision logic simplifies to combining discriminative local patterns via logical disjunction (OR). Formally, if $\mathcal{P}_y \subseteq \mathcal{P}$ denotes the subset of predicates associated with label $y$, the decision rule is defined as $\Phi_y(v) = \bigvee_{p_j \in \mathcal{P}_y} p_j(v)$.

\vspace{-5pt}

% \paragraph{Remark: Explanatory Precision and Utility.}

% This topological, role-based formulation of local predicates significantly extends the expressivity of standard GML. By resolving neighborhoods into granular orbits, our approach captures strict topological constraints that GML often fails to distinguish. Crucially, this formalization addresses a fundamental misalignment in current XAI. While motif-based methods~\cite{xgnn, GNNInterpreter} effectively isolate \emph{influential} subgraphs, they produce only subgraph-level visualizations that lack explicit semantic definitions. This is often insufficient for scientific discovery, where experts reason in terms of precise rules and structural constraints, such as SMARTS-style specifications~\cite{daylight2019smarts} regarding allowable atom types and bonding configurations. Such strict constraints cannot be revealed by individual subgraph explanations alone. Our framework resolves this by grounding elements in logical predicates that mirror these expert-defined rules, making the entire decision mechanism explicit~\cite{LogicXGNN}. For instance, as illustrated in Figure~\ref{fig:orbit_rule_eg}, our learned predicates align closely with SMARTS patterns used in Structure-Activity Relationship analysis, offering a direct bridge between learned representations and chemical intuition. We further analyze this alignment in Appendix~\ref{app:smarts_details}.

\begin{figure*}[th] 
    \centering
    \includegraphics[width=\linewidth]{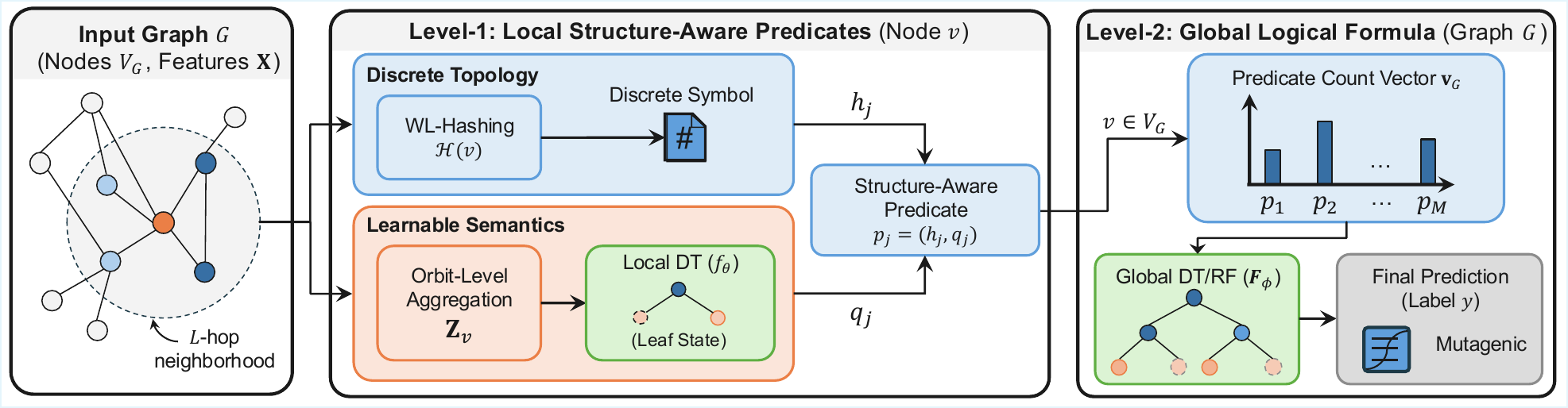}
    \caption{An overview of the \textsc{SymGraph} framework.}
    \label{fig:overview}
\end{figure*}

\section{The \textsc{SymGraph} Framework}
\label{sec:method}

\paragraph{Granular Structural Awareness via Hashing.}
To establish a standalone symbolic framework that parallels the reasoning depth of GNNs, we explicitly model the local scope of interaction. Mimicking the architectural inductive bias of an $L$-layer GNN, we focus on the receptive field of a node $v$—specifically, the $L$-hop subgraph $S_v$ centered at $v$.

However, rather than relying on continuous aggregation to implicitly encode this structure (which often leads to information loss), we explicitly capture its discrete topological signature using a structural hashing function.\footnote{We use Weisfeiler--Lehman hashing to map isomorphic subgraphs to identical identifiers, thereby capturing their purely structural skeleton.} Formally, the structural pattern of a node $v$ is identified by:
\begin{equation}
    \mathcal{H}(v) = \text{Hash}(S_v)
\end{equation}
This allows our model to reason about specific topological contexts as distinct symbols, mimicking the spatial scope of message passing while explicitly retaining structural information for superior expressiveness and interpretability.

\vspace{-5pt}

\paragraph{Topological Role-based Aggregation via Orbits.}
With the topological skeleton $\mathcal{H}(v)$ established, we must aggregate the input features $\mathbf{X}$ over $S_v$ in a manner that strictly respects structural symmetries—providing a precise alternative to the symmetric summation of GNNs. We achieve this by analyzing the automorphism group $\text{Aut}(S_v)$ acting on the node set $V_S$. The orbit of a node $u \in V_S$ is defined as:
\begin{equation}
    \text{Orb}(u) = \{ w \in V_S \mid \exists \, \pi \in \text{Aut}(S_v) \text{ such that } \pi(u) = w \}
\end{equation}
Each orbit represents an equivalence class of nodes that are structurally indistinguishable within the subgraph. To ensure a canonical feature representation, we partition the subgraph nodes into orbits and establish a consistent ordering $\mathcal{O}(S_v)$ (see Algorithm \ref{alg:stable_orbits} in Appendix~\ref{app:topo_role_encoding_details}).

\begin{definition}[Orbit-Aware Feature Vector $\mathbf{Z}_v$]
\label{def:orbit_feature}
Based on the canonical orbit ordering $\mathcal{O}(S_v)$, we construct the feature vector $\mathbf{Z}_{v}$ by concatenating orbit-wise aggregations:
\begin{equation}
    \mathbf{Z}_{v} = \text{CONCAT}_{\text{Orb} \in \mathcal{O}(S_v)} \Big( \text{AGGREGATE}_{u \in \text{Orb}} (\mathbf{X}_u) \Big)
\end{equation}
where $\text{AGGREGATE}$ applies histogram encoding for discrete features and mean-pooling for continuous features \cite{LogicXGNN}. This vector $\mathbf{Z}_v$ explicitly preserves the feature distribution relative to topological roles, preventing the ``feature mixing'' common in standard message passing. 
\end{definition}

\paragraph{Learnable Local Structural Predicates (Level 1).}
While hashing captures the topological skeleton, the semantic content is encoded in the node features. To model the explicit impact of these semantic features on the decision outcome, we employ a learnable function $f_{\theta}$ that maps the orbit-aware feature vector $\mathbf{Z}_v$ directly to a discrete state $q$. Crucially, $q$ acts as a \emph{local discriminative signal}, transforming continuous features into symbolic inputs that drive downstream decisions. This abstraction mimics the information aggregation of message passing in GNNs.

To maximize interpretability, we instantiate $f_{\theta}$ using off-the-shelf Decision Trees (DT), where the discrete state $q$ corresponds directly to the \emph{leaf node} reached by $\mathbf{Z}_v$. Unlike opaque neural embeddings, this mechanism offers inherent transparency: the path to each leaf constitutes an explicit boolean conjunction of feature thresholds (e.g., $x_0 > \alpha \land x_1 < \beta$). This effectively decomposes the decision boundary into a DNF rule, ensuring precise alignment with the logical formalisms established in Section \ref{sec:logic_prelim}.

\paragraph{Core Mechanism: The Structure-Aware Predicate.}
Our central contribution is the reformulation of a predicate. We do not treat predicates as simple boolean flags on a structure. Instead, we define them as \emph{composite tuples} that fuse topology with learned semantic states.

\begin{definition}[Structure-Aware Predicate]
\label{def:predicate}
We define a predicate $p_j$ by unifying a specific topological hash $h_{j}$ and a specific semantic state $q_{j}$ into a unique tuple:
\begin{equation}
    \boxed{ p_{j} = (h_{j}, q_{j}) }
\end{equation}
The logical evaluation of a node $v$ against predicate $p_j$ is a binary truth value determined by:
\begin{equation}
    \text{Eval}(v, p_{j}) = \underbrace{\mathbb{I}\left( \mathcal{H}(v) = h_{j} \right)}_{\text{Structure Condition}} \cdot \underbrace{\mathbb{I}\left( f_{\theta}(\mathbf{Z}_v) = q_{j} \right)}_{\text{Feature Condition}}
\end{equation}
The predicate evaluates to $1$ (True) if and only if both the topological structure of node $v$'s $L$-hop neighborhood matches the hash $h_{j}$ and the feature logic maps the input $\mathbf{Z}_v$ to the corresponding state $q_{j}$.
\end{definition}

\textit{Remark: Sufficiency for Node Classification.} For node-level tasks, this local symbolic abstraction is sufficient. We provide further details and results in Appendix~\ref{app:node_cls}.

% Since the state $q$ serves as a discriminative signal learned via direct supervision, we simplify the decision process to a logical disjunction: $\Phi_y(v) = \bigvee_{p_j \in \mathcal{P}_y} p_j(v)$.

\paragraph{Theoretical Analysis: Local Expressiveness.}
Our formulation provides strictly higher expressiveness than standard Message 
Passing Neural Networks. Standard MPNNs update a node's embedding by aggregating 
neighbors symmetrically (e.g., via summation), which inherently loses information 
about \textit{where} specific features are located relative to the node's 
topological skeleton. Specifically, our predicates distinguish graph classes 
where 1-WL fails, including regular graphs and graphs with permuted feature 
distributions across structural orbits.

% \paragraph{Theoretical Analysis: Local Expressiveness.}
% Our formulation provides strictly higher expressiveness than standard Message 
% Passing Neural Networks. Standard MPNNs aggregate neighbors symmetrically 
% (e.g., via summation), losing information about \textit{where} specific 
% features are located relative to the topological skeleton. Specifically, 
% our predicates distinguish graph classes where 1-WL fails, including regular 
% graphs and graphs with permuted feature distributions across structural orbits.

\begin{proposition}
\label{prop:local_expressiveness}
The tuple predicate $p_j = (h_{j}, q_{j})$ distinguishes rooted subgraphs that are indistinguishable by the standard 1-WL test (and consequently standard MPNNs), specifically in cases where: (1) topological structures differ but generate identical 1-WL color refinements, or (2) topological structures are identical but feature distributions differ across non-equivalent structural orbits.
\end{proposition}

\begin{proof}
We provide the formal proof in Appendix~\ref{app:proof_local}.
\end{proof}
% \begin{proof}[Proof Sketch] Standard MPNNs conflate structure and features into a single pooled vector, which can map distinct neighborhoods to identical embeddings. In contrast, our approach guarantees strictly higher expressiveness by (1) explicitly identifying the receptive field structure (resolving topological ambiguities), and (2) aggregating features separately per structural orbit (retaining granular feature position information that global summation discards). We provide the formal proof in Appendix~\ref{app:proof_local}.
% \end{proof}

\subsection{Global Aggregation via Structural Histograms}
\label{sec:global_aggregation}
\paragraph{The ``Bag-of-Vectors'' Limitation.}
Standard GNNs typically employ a global readout function, such as mean-pooling, to generate a graph-level embedding. This operation reduces the graph to a ``Bag-of-Vectors'', where the distinct identities of local structures are compressed into a single latent vector. Consequently, distinct combinations of substructures can become indistinguishable due to destructive interference in the continuous embedding space (e.g., a ``negative'' feature cancelling out a ``positive'' feature).

To resolve this, we propose a paradigm shift: moving from an opaque ``Bag-of-Vectors'' to an explicit and traceable \emph{``Bag-of-Predicates''}. By treating the set of structure-aware predicates as a global vocabulary, we preserve both the structural topology and the semantic attributes of local patterns, preventing information loss during aggregation.

% Let $\mathcal{D} = \{G_1, \dots, G_N\}$ be the training dataset.

\begin{definition}[Global Predicate Vocabulary $\mathcal{P}$] 
\label{def:vocab}
We construct a global schema $\mathcal{P}$ containing all unique structure-aware predicates discovered in the training data $\mathcal{D}$:
\begin{equation}
    \mathcal{P} = \bigcup_{G \in \mathcal{D}} \bigcup_{v \in V_G} \{ p \mid \text{Eval}(v, p) = 1 \}
\end{equation}
In other words, $\mathcal{P}$ is the Cartesian product of the set of unique topological hashes and the discrete semantic states observed in the training set. Let $J = |\mathcal{P}|$ be the size of this vocabulary. We assign a fixed index $j \in \{1, \dots, J\}$ to each unique predicate $p_j \in \mathcal{P}$, establishing a canonical ordering. 
\end{definition}

\paragraph{Tabular Graph Representation.}
Using this global schema, we map every graph $G$ to a fixed-size vector. Unlike standard pooling, which only aggregates implicit \textit{semantic} values, this mapping explicitly aggregates predicate counts.

\begin{definition}[Predicate Count Vector $\mathbf{v}_G$]
\label{def:count_vector}
For a graph $G$, the Predicate Count Vector $\mathbf{v}_G \in \mathbb{N}^J$ is a histogram vector where the $j$-th entry represents the frequency of the $j$-th predicate in the graph:
\begin{equation}
    \mathbf{v}_{G}[j] = \sum_{v \in V_G} \text{Eval}(v, p_j)
\end{equation}
This transformation effectively converts the graph into a ``Bag-of-Predicates'' representation, capturing not just which structural and semantic patterns exist, but their multiplicity.
\end{definition}

\paragraph{Learnable Global Logic (Level 2).}
Once the graph is represented as $\mathbf{v}_G$, the classification task becomes learning a function $F_{\phi}: \mathbb{N}^J \rightarrow \mathcal{Y}$ that maps structural frequencies to graph labels. Crucially, because $\mathbf{v}_G$ consists of disentangled, semantically meaningful counts, the choice of $F_{\phi}$ dictates the nature of the learned logic and its interpretability.

To ensure the global reasoning remains symbolic and hierarchical, we instantiate $F_{\phi}$ as a Decision Tree. This architecture directly realizes the Q-DNF logic formalized in Section \ref{sec:logic_prelim}. By optimizing paths to leaf nodes, the tree naturally discovers the global formula $\Phi(G)$ as a disjunction of conjunctive rules defined by structural count thresholds (e.g., $C_{\text{Ring}} \ge 2 \land C_{\text{Star}} = 0$), enabling transparent, white-box reasoning. Alternatively, for better generalization performance, $F_{\phi}$ can be instantiated as a Random Forest (RF). Although this functions as a ``gray-box'' ensemble, it retains interpretability by identifying discriminative predicates via feature importance metrics, which is a popular explanation paradigm in XAI \cite{GCneuron, PGExplainer}.

% While this functions as a ``gray-box'' ensemble, it preserves the symbolic nature of the input predicates and provides significant interpretability by identifying the most discriminative  predicates via feature importance metrics.

% Alternatively, for better generalization, $F_{\phi}$ can be instantiated as a Random Forest (RF). Although a ``gray-box,'' it retains interpretability by identifying discriminative symbolic predicates via feature importance metrics.

% Alternatively, for better generalization performance, $F_{\phi}$ can be instantiated as a Random Forest (RF), which functions as a ``gray-box'' ensemble while preserving the symbolic nature of the input predicates.

\paragraph{Theoretical Analysis: Global Expressiveness.}
The transition from continuous pooling to discrete histograms offers a distinct theoretical advantage: superior expressiveness regarding substructure counting.

\begin{proposition}
\label{prop:global_expressiveness}
The Predicate Count Vector $\mathbf{v}_G$ is strictly more expressive than standard global aggregation functions (e.g., Min, Max, or Mean pooling), particularly in its ability to distinguish graphs based on substructure multiplicity.
\end{proposition}

\begin{proof}
We provide the formal proof in Appendix~\ref{app:proof_global}.
\end{proof}

% \begin{proof}[Proof Sketch]
% Standard global aggregation functions suffer from mathematical limitations when counting substructures. \textit{Max-} and \textit{Min-pooling} are idempotent operations ($\max(x, x) = x$), rendering them incapable of distinguishing a single motif from multiple copies. \textit{Mean-pooling} effectively dilutes local signals by normalizing over the graph size. In contrast, $\mathbf{v}_G$ functions as a structural histogram, summing logical evaluations in orthogonal dimensions to strictly preserve multiplicity without interference or normalization artifacts. We provide the formal proof and counter-examples in Appendix~\ref{app:proof_global}.
% \end{proof}

\begin{table*}[th]
\centering
\caption{\textbf{Performance Benchmark.} Comparison of test accuracy (\%) and total runtime ($10^2$ seconds) against state-of-the-art self-explainable GNNs and standard message-passing GNN baselines. Results are averaged over five random seeds. \textbf{Bold} and \underline{underlined} indicate the best and second-best performance, respectively. \textsc{SymGraph}-DT and \textsc{SymGraph}-RF denote the decision tree and random forest variants of our framework. Notably, our approach achieves superior accuracy on complex molecular datasets while being \emph{orders of magnitude faster} than competing self-explainable methods like \textsc{LogiX-GIN} and \textsc{PiGNN}. See more results in Appendix \ref{app:add_results}.}
\label{tab:graph_acc}
\vspace{-5pt}

\resizebox{\textwidth}{!}{%
\begin{tabular}{@{}l rr rr rr rr rr rr@{}}
\toprule
& \multicolumn{2}{c}{\textbf{Ba2Motifs}} 
& \multicolumn{2}{c}{\textbf{BaMultiShapes}} 
& \multicolumn{2}{c}{\textbf{PROTEINS}} 
& \multicolumn{2}{c}{\textbf{Mutagenicity}} 
& \multicolumn{2}{c}{\textbf{NCI1}} 
& \multicolumn{2}{c}{\textbf{BBBP}} \\

\cmidrule(lr){2-3} \cmidrule(lr){4-5} \cmidrule(lr){6-7} \cmidrule(lr){8-9} \cmidrule(lr){10-11} \cmidrule(lr){12-13}

\textbf{Model} 
& Acc$\uparrow$ & Time$\downarrow$ 
& Acc$\uparrow$ & Time$\downarrow$ 
& Acc$\uparrow$ & Time$\downarrow$ 
& Acc$\uparrow$ & Time$\downarrow$ 
& Acc$\uparrow$ & Time$\downarrow$ 
& Acc$\uparrow$ & Time$\downarrow$ \\
\midrule

\textsc{GCN}
& 58.3 $\pm$ 2.3  & \underline{0.2 $\pm$ 0.0}
& 70.7 $\pm$ 9.7 & 0.5 $\pm$ 0.0
& 71.1 $\pm$ 1.5 & 2.8 $\pm$ 0.4
& 81.6 $\pm$ 1.0 & 14.4 $\pm$ 1.6
& 76.0 $\pm$ 1.3 & 12.6 $\pm$ 1.1
& 86.1 $\pm$ 1.0 & 5.3 $\pm$ 1.0 \\

\textsc{GraphSAGE}
& 48.2 $\pm$ 0.8 & 0.3 $\pm$ 0.0
& 50.6 $\pm$ 2.1 & \underline{0.3 $\pm$ 0.0}
& \underline{72.3 $\pm$ 2.3} & 2.2 $\pm$ 1.1
& 81.9 $\pm$ 1.6 & 13.9 $\pm$ 3.2
& 79.6 $\pm$ 0.3 & 12.9 $\pm$ 1.4
& 86.9 $\pm$ 1.0& \underline{4.6 $\pm$ 1.7}\\

\textsc{GAT}
& 48.2 $\pm$ 0.8 & 0.3 $\pm$ 0.0
& 49.4 $\pm$ 2.1 & \underline{0.3 $\pm$ 0.0}
& 70.9 $\pm$ 1.0 & \underline{1.8 $\pm$ 1.1}
& 82.2 $\pm$ 1.2 & 18.2 $\pm$ 1.9
& 78.0 $\pm$ 2.0 & 16.4 $\pm$ 4.3
& \underline{88.2 $\pm$ 2.0} & 6.9 $\pm$ 1.0 \\

\textsc{GIN}
& \textbf{100.0 $\pm$ 0.0} & 0.3 $\pm$ 0.0
& 95.0 $\pm$ 1.5 & 0.4 $\pm$ 0.2
& 68.5 $\pm$ 2.3 & 3.0 $\pm$ 0.4
& \underline{83.0 $\pm$ 1.6} & 13.1 $\pm$ 1.4
& \underline{81.2 $\pm$ 1.3} & 11.9 $\pm$ 2.2
& 87.7 $\pm$ 1.0 & 5.5 $\pm$ 0.6 \\

% \midrule

% & 50.0 $\pm$ 0.0 & 51.8 $\pm$ 0.2
% & 50.0 $\pm$ 0.0&50.4 $\pm$ 0.2
% & 59.6 $\pm$ 0.0& 65.2 $\pm$ 0.8
% & 55.4 $\pm$  0.1&163.6  $\pm$  2.7
% & 50.0$\pm$ 0.0&326.7 $\pm$ 13.1
% & 23.5$\pm$ 0.0&81.8 $\pm$ 2.4\\

\textsc{GNAN}
& 49.5 $\pm$ 0.5 & 51.8 $\pm$ 1.2
& 51.4 $\pm$ 1.6 & 50.4 $\pm$ 1.2
& 59.6 $\pm$ 0.0 & 65.2 $\pm$ 1.8
& 55.4 $\pm$ 0.1 & 163.6 $\pm$ 5.7
& 52.1 $\pm$ 1.7 & 326.7 $\pm$ 13.1
& 23.5 $\pm$ 0.0 & 81.8 $\pm$ 3.4\\

\textsc{GIB}
& \textbf{100.0 $\pm$ 0.0} & 31.5 $\pm$ 0.7
& 95.8 $\pm$ 4.2 & 38.8 $\pm$ 0.8
& 71.8 $\pm$ 6.1 & 21.2 $\pm$ 0.6
& 72.4 $\pm$ 7.1 & 83.0 $\pm$ 1.7
& 75.1 $\pm$ 3.1 & 137.1 $\pm$ 6.8
& 84.6 $\pm$ 1.0 & 41.2 $\pm$ 2.3 \\

\textsc{KerGNN}
&97.9 $\pm$ 0.9 & 2.7 $\pm$ 0.2
&82.7 $\pm$ 2.1 & 2.9 $\pm$ 0.2
&70.6 $\pm$ 5.1 & 3.2 $\pm$ 0.2
&72.2 $\pm$ 3.0 & \underline{10.8 $\pm$ 1.3}
& 67.4 $\pm$ 1.7 & \underline{10.1 $\pm$ 1.2}
&82.4 $\pm$ 2.7 & 5.1 $\pm$ 0.5\\

\textsc{PiGNN}
& \underline{99.4 $\pm$ 0.8} & 20.9 $\pm$ 0.7
& 85.4 $\pm$ 5.1 & 21.4 $\pm$ 1.5
& 70.0 $\pm$ 2.4 & 10.2 $\pm$ 1.1
& 78.8 $\pm$ 3.0 & 93.3 $\pm$ 3.8
& 76.6 $\pm$ 1.5 & 86.8 $\pm$ 3.5
&84.2 $\pm$ 2.0 & 43.9 $\pm$ 1.9 \\

\textsc{LogiX-GIN}
& \textbf{100.0 $\pm$ 0.0} & 34.4 $\pm$ 1.3
& \textbf{100.0 $\pm$ 0.0} & 35.0 $\pm$ 1.5
& 72.1 $\pm$ 5.8 &73.5 $\pm$ 2.4
& 79.0 $\pm$ 1.6 & 72.6 $\pm$ 2.5
& 77.6  $\pm$ 2.8  & 158.4  $\pm$ 14.9
& 85.1 $\pm$ 0.7 & 77.4 $\pm$ 3.1 \\

\midrule

\textsc{SymGraph}-DT
& \textbf{100.0 $\pm$ 0.0} & \multirow{2}{*}{\textbf{0.1 $\pm$ 0.0}}
& \underline{97.5 $\pm$ 0.5} & \multirow{2}{*}{\textbf{0.1 $\pm$ 0.0}}
& \underline{72.9 $\pm$ 4.1} & \multirow{2}{*}{\textbf{1.7 $\pm$ 0.1}}
& 81.4 $\pm$ 1.0 & \multirow{2}{*}{\textbf{3.5 $\pm$ 0.3}}
& 76.8 $\pm$ 1.6 & \multirow{2}{*}{\textbf{3.8 $\pm$ 0.2}}
& 87.1 $\pm$ 1.2 & \multirow{2}{*}{\textbf{1.3 $\pm$ 0.1}} \\

\textsc{SymGraph}-RF
& \textbf{100.0 $\pm$ 0.0} &
& 96.7  $\pm$ 0.8 &
& \textbf{78.9  $\pm$ 1.8} &
& \textbf{84.8 $\pm$ 1.1} &
& \textbf{84.7 $\pm$ 1.1} &
& \textbf{89.5 $\pm$ 1.7} \\

\bottomrule
\end{tabular}
}
\end{table*}

\vspace{-5pt}

\subsection{Optimization via Evolutionary Search}
\label{sec:training}

\paragraph{Leaf State Granularity as Optimization Goal.}
The expressive power of \textsc{SymGraph} is governed by the granularity of the local predicates $\mathcal{P}$, specifically the leaf node resolution of each topological pattern.\footnote{Note that for tasks dominated by \emph{strong structural signals}, feature quantization ($q$) often becomes redundant. In this regime, the optimization effectively simplifies to standard supervised learning (e.g., using a DT) over a fixed vocabulary of topological hashes.} While the quantization states $q$ can be learned via direct supervision from labels, the final graph-level decision relies on the global Predicate Count Vector $\mathbf{v}_G$. Consequently, each topological hash receives only weak supervision, necessitating a search for the optimal semantic granularity for its leaf states.

We therefore formulate the optimization of \textsc{SymGraph} as a combinatorial search for the optimal vocabulary structure. We parameterize the semantic resolution via a genome vector $\boldsymbol{\lambda} \in \{1, \dots, \Lambda_{\max}\}^K$, where $K$ is the number of unique topological patterns in $\mathcal{D}$ and $\lambda_k$ dictates the \emph{active leaf count} (granularity level) for the topological pattern $h_k$. Our objective is to discover the configuration $\boldsymbol{\lambda}^*$ that maximizes global fitness while penalizing vocabulary inflation:
\begin{equation}
\label{eq:optimization_goal}
    \boldsymbol{\lambda}^* = \argmax_{\boldsymbol{\lambda}} \Big[ \Pi\big( F_{\phi}(\mathbf{V}(\mathcal{D}, \boldsymbol{\lambda})) \big) - \gamma \|\boldsymbol{\lambda}\|_1 \Big]
\end{equation}
where $\mathbf{V}(\mathcal{D}, \boldsymbol{\lambda})$ represents the dataset projected into the feature space defined by $\boldsymbol{\lambda}$ (i.e., the matrix of vectors $\mathbf{v}_G$), $\Pi(\cdot)$ denotes the validation accuracy of the global classifier, and $\gamma$ controls the sparsity of the symbolic representation.

\paragraph{Evolutionary Search with Master Tree Pre-Caching.}
Since the search space of $\boldsymbol{\lambda}$ is strictly combinatorial, we employ a Genetic Algorithm (GA) to traverse it efficiently. However, a na\"ive execution is still computationally prohibitive, as evaluating the fitness of any candidate genome implies retraining thousands of local models. To resolve this evaluation bottleneck, we introduce a \emph{Master Tree Pre-Caching} strategy. First, we impose a uniform complexity upper bound $\Lambda_{\max}$ (maximum allowable leaf count) across all patterns. For each topological hash $h_k$, we then pre-train a single \emph{Master Decision Tree} expanded to this limit, capturing the finest possible semantic granularity. We cache the leaf index assignments for every node $v$ across all pruning levels $1 \dots \Lambda_{\max}$ into a \emph{lookup table matrix} $\mathbf{T}_k$. This transforms the generation of semantic states into an $O(1)$ lookup operation, $q(\lambda_k) = \mathbf{T}_k[\text{idx}(v), \lambda_k]$, effectively allowing the global vocabulary $\mathcal{P}$ to be mutated dynamically without model retraining. By coupling this caching mechanism with evolutionary search, we can efficiently discover the optimal semantic abstraction for each topological structure. The full algorithm is detailed in Appendix~\ref{app:ga_algo}.

\section{Experiments}
\label{sec:experiments}

In this section, we conduct extensive experimental evaluations to answer the following research questions:

\begin{itemize}
    \item \textbf{RQ1 (Performance):} How does \textsc{SymGraph} compare to SOTA self-explainable GNNs and standard GNN baselines in classification accuracy and efficiency?

    \item \textbf{RQ2 (Ablation):} How do key design choices, such as the Topological Role Encoding and the Predicate Counting, contribute to the model's overall efficacy?

    \item \textbf{RQ3 (Discovery \& Comparison):} Can \textsc{SymGraph} uncover scientifically meaningful patterns that align with expert domain knowledge, particularly in complex chemical tasks? Furthermore, how does the quality of its explanations compare to those of existing methods?
\end{itemize}

\paragraph{Baselines and Datasets.}

We benchmark \textsc{SymGraph} against a comprehensive set of competitive methods, categorized into two groups: (i) State-of-the-art self-explainable GNNs, including \textsc{LogiX-GIN}~\cite{LogiXGIN}, \textsc{PiGNN}~\cite{pignn}, \textsc{GIB}~\cite{gib}, \textsc{KerGNN}~\cite{kergnn}, and \textsc{GNAN}~\cite{gnan}; and (ii) Standard Message Passing Neural Networks, specifically GCN, GIN, GraphSAGE~\cite{sage}, and GAT~\cite{GAT}. To ensure a rigorous comparison, we adopt the exact evaluation suite used in \textsc{LogiX-GIN}~\cite{LogiXGIN} and prior literature. This includes diverse standard XAI benchmarks such as BaMultiShapes, as well as real-world molecular datasets like Mutagenicity. This selection reflects a dual focus: overcoming the 1-WL expressivity bottleneck while demonstrating actionable Explainable AI in high-stakes domains such as drug discovery.
Detailed descriptions of these datasets are provided in Appendix~\ref{app:dataset}.

For all baselines, we adopt the hyperparameters recommended in the original works to ensure fair comparison. Detailed training and implementation configurations for \textsc{SymGraph} and baselines are provided in Appendix~\ref{app:setup}.

\begin{figure*}[b]
    \centering % Center the figures
    % --- Top Row ---
    \subfigure[\textsc{LogiX-GIN} learned irrelevant chemical components and rules.]{%
        \includegraphics[width=0.48\linewidth]{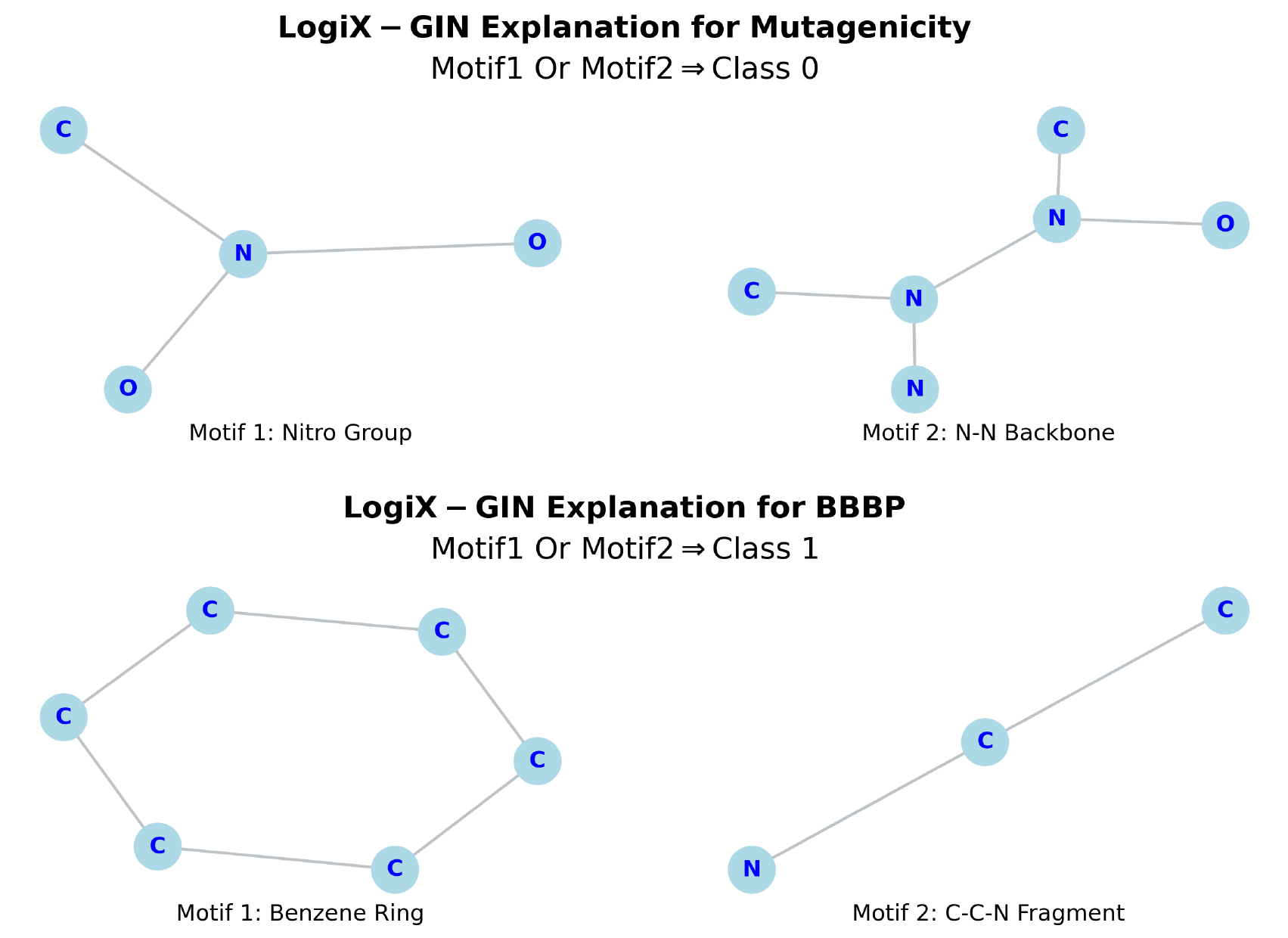}
        \label{fig:top-left}
        \vspace{5pt}
    }
    \subfigure[\textsc{GraphTrail} learned invalid chemical components.]{%
        \includegraphics[width=0.48\linewidth]{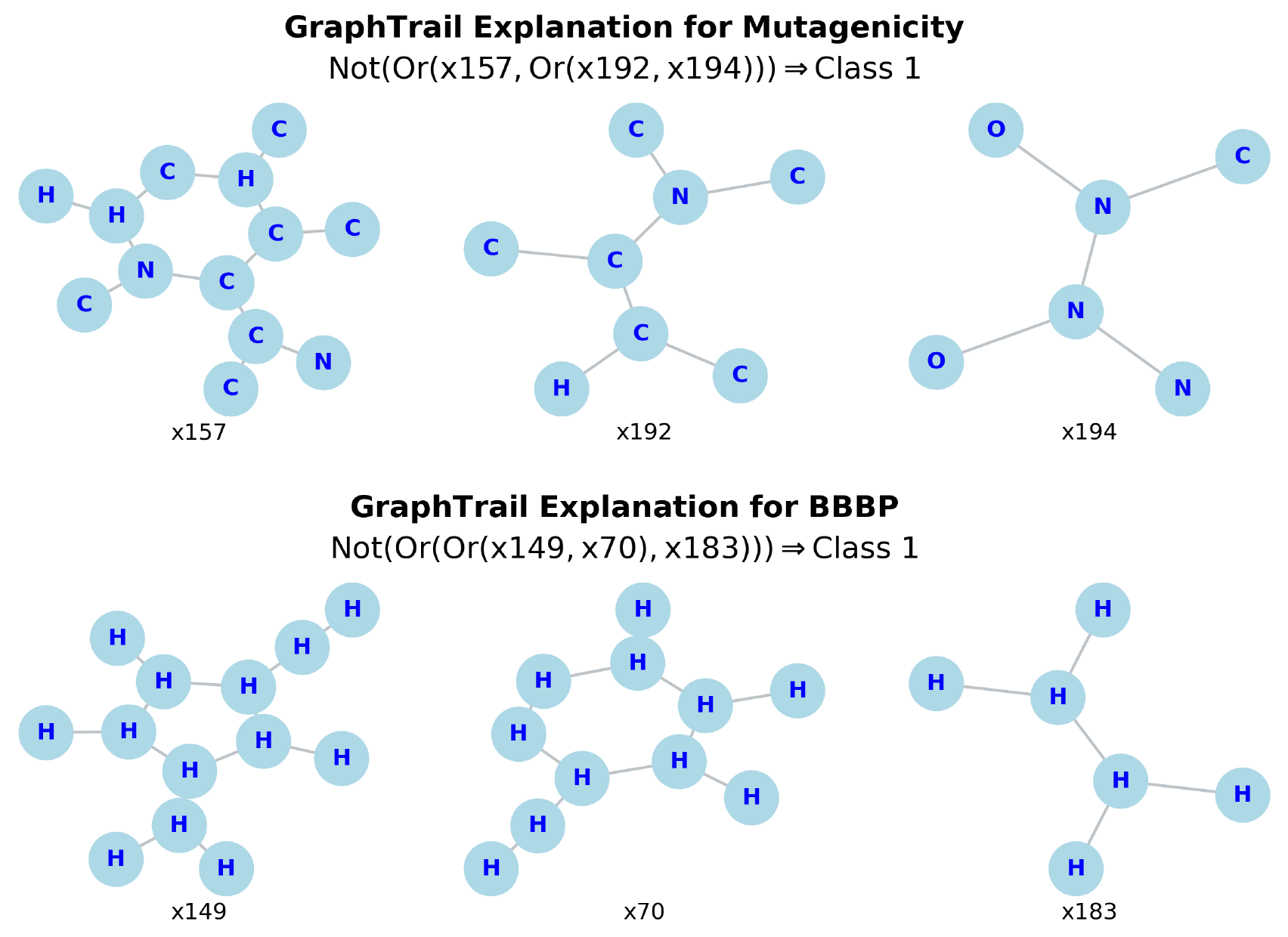}
        \label{fig:top-right}
        \vspace{5pt}
    }
    % --- Overall Caption ---
    \vspace{-5pt}
    \caption{Baselines' explanations exhibit irrelevant rules and chemically invalid components.}
    \label{fig:baseline_exp}
    \vspace{-5pt}
\end{figure*}

% \begin{figure*}[t]
%     \centering % Center the figures
%     % --- Top Row ---
%     \subfigure[\textsc{GLG} learned no rules or conflicted rules.]{%
%         \includegraphics[width=0.48\linewidth]{figures/merged_explanations_glg.pdf}
%         \label{fig:top-left}
%         \vspace{5pt}
%     }
%     \subfigure[\textsc{GTrail} learned invalid chemical subgraphs.]{%
%         \includegraphics[width=0.48\linewidth]{figures/merged_explanations_gt.pdf}
%         \label{fig:top-right}
%         \vspace{5pt}
%     }
%     % --- Overall Caption ---
%     \caption{Baselines' explanations exhibit conflicting rules and chemically invalid subgraphs.}
%     \label{fig:baseline_exp}
% \end{figure*}

\begin{figure*}[t]
    \centering % Center the figures
    % --- Top Row ---
    \subfigure[
    \vspace{5pt}
    \begin{tabular}{@{}l@{}}
    % $(\neg p_{3} \land p_{7}) \lor (p_{3} \land p_{27}) \Rightarrow \text{BBBP class } 0$ \\
    % $(\neg p_{3} \land \neg p_{7}) \lor (p_{3} \land \neg p_{27}) \Rightarrow \text{BBBP class } 1$
    \end{tabular}
    ]{%
        \includegraphics[width=0.48\linewidth]{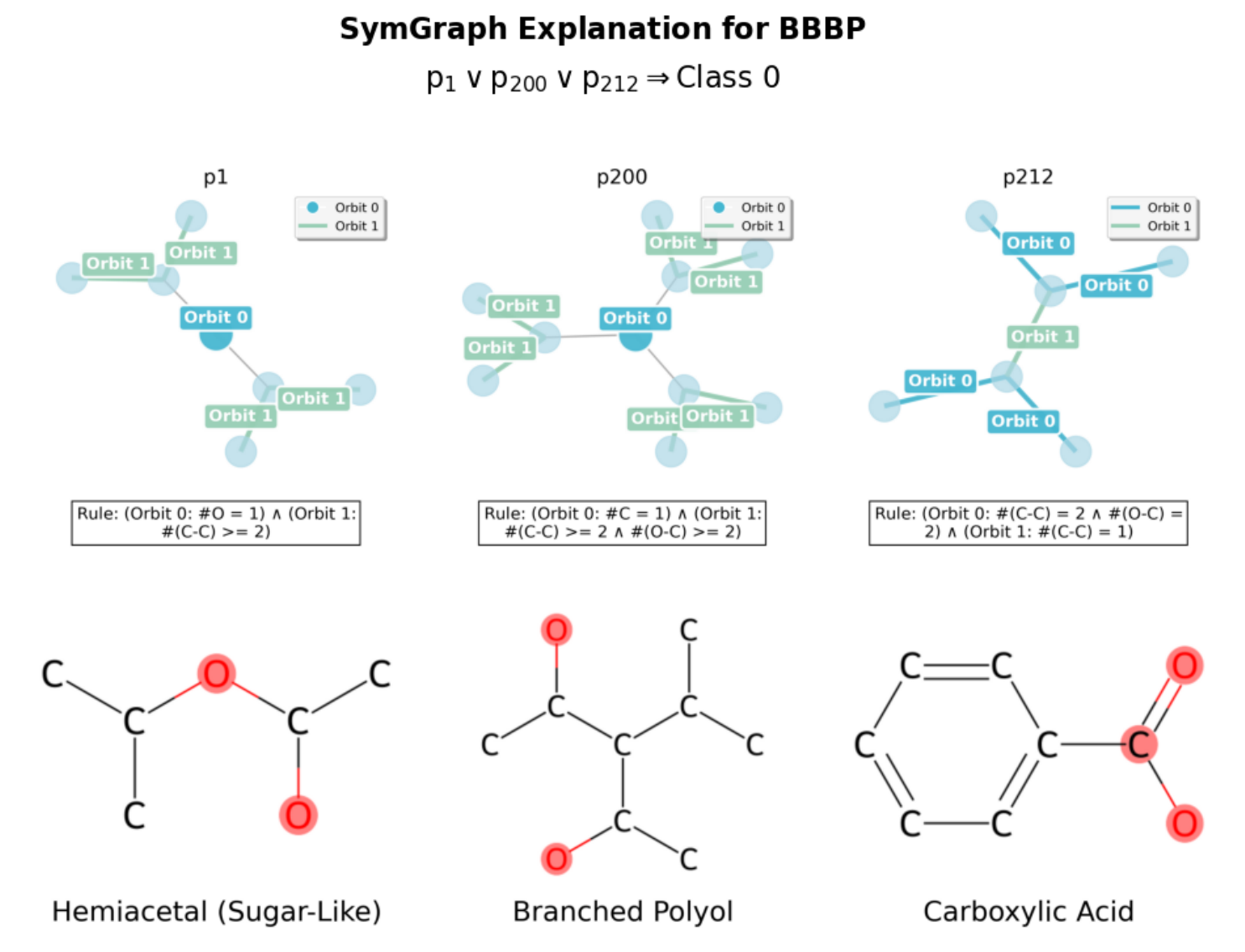}
        \label{fig:top-left}
        \vspace{5pt}
    }
    \subfigure[
    \vspace{5pt}
    \begin{tabular}{@{}l@{}}
    % $(\neg p_{2} \land p_{20}) \lor (p_{2} \land \neg p_{38}) \Rightarrow \text{Mutag. class } 0$ \\
    % $(\neg p_{2} \land \neg p_{20}) \lor (p_{2} \land p_{38}) \Rightarrow \text{Mutag. class } 1$
    \end{tabular}
    ]{%
        \includegraphics[width=0.48\linewidth]{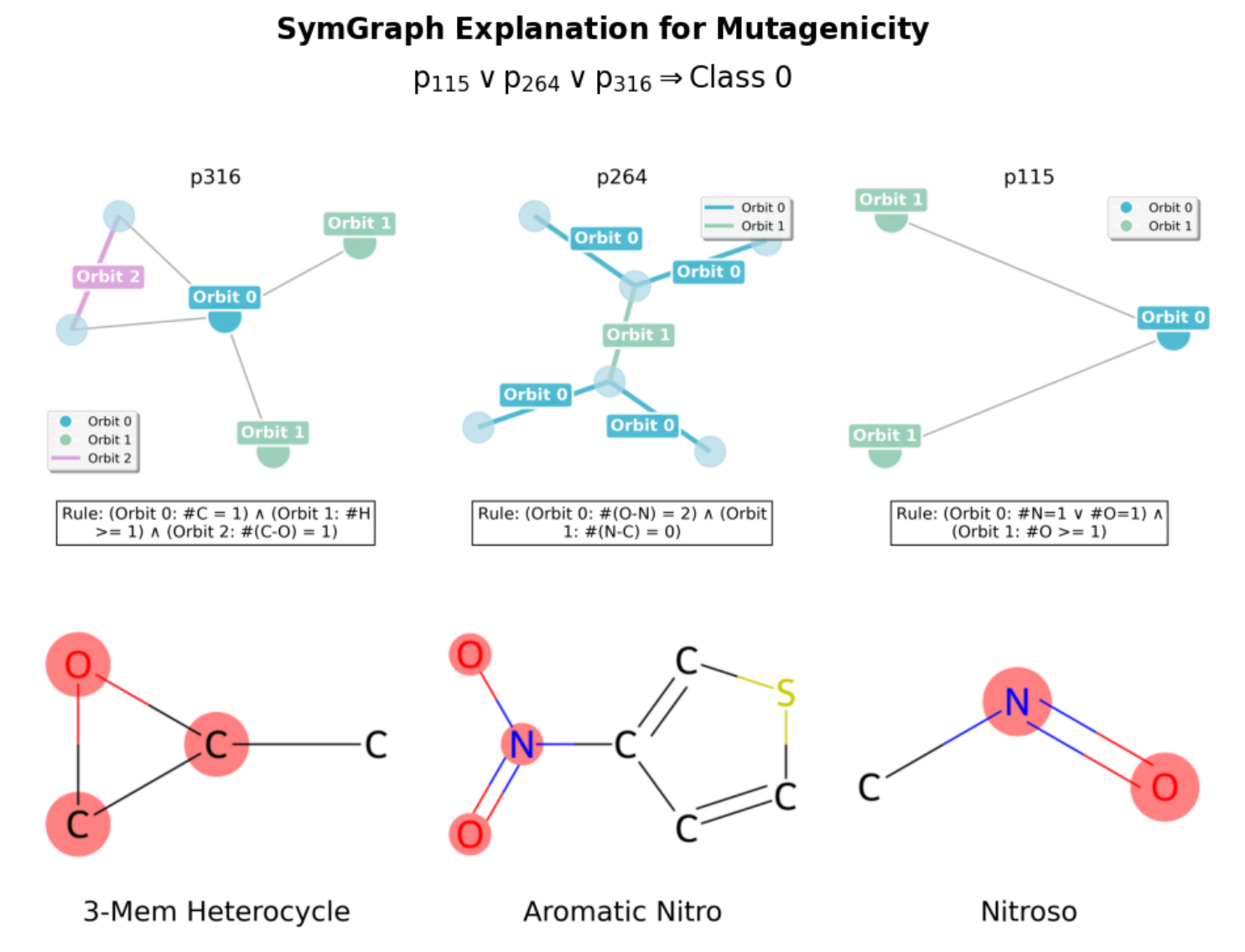}
        \label{fig:top-right}
        \vspace{5pt}
    }
    % --- Overall Caption ---
    \vspace{-5pt}
    \caption{\textsc{SymGraph} identifies expert-validated chemical motifs that provide ground-truth explanations for molecular properties. Our predicates enforce structural constraints on node and edge attributes, which are readily mapped to SMARTS patterns, as illustrated by the red highlights on the chemical compounds. This provides a direct bridge between learned logic and established domain knowledge.}
    \label{fig:our_exp}
    \vspace{-5pt}
\end{figure*}

\subsection{How effective is \textsc{SymGraph} as a classification tool?}
\label{sec:exp_performance}

We evaluate the classification performance of \textsc{SymGraph} against state-of-the-art baselines in Table~\ref{tab:graph_acc}. Our method demonstrates both superior accuracy and exceptional computational efficiency across diverse benchmark tasks.

On molecular datasets, \textsc{SymGraph} establishes a new state-of-the-art for interpretable graph classification, surpassing the leading self-explainable baseline \textsc{LogiX-GIN} by substantial average margins of approximately 6\%. More strikingly, \textsc{SymGraph}-RF consistently outperforms \textsc{GIN} across all datasets, despite \textsc{GIN}'s theoretical equivalence to the 1-WL test—the discriminative limit of standard message-passing neural networks. This systematic improvement provides empirical evidence that our structure-aware predicates capture higher-order topological patterns (e.g., functional group symmetries, ring structures) that are provably beyond the expressive capacity of standard MPNNs.

% On synthetic benchmarks, \textsc{SymGraph} achieves 100\% accuracy on Ba2Motifs and 97.5\% on BaMultiShapes. The marginal gap compared to \textsc{LogiX-GIN} on BaMultiShapes reflects a fundamental trade-off: \textsc{LogiX-GIN} uses Graded Modal Logic \cite{DBLP:journals/sLogica/Rijke00} to aggregate neighborhood degree distributions, which enables broader generalization but risks ``structural hallucination'', i.e, misclassifying trees as houses based on degree sequences alone. In contrast, \textsc{SymGraph} enforces exact subgraph isomorphism, ensuring explanations are structurally equivalent to ground-truth patterns rather than merely statistically correlated.

\begin{table}[t]
    \centering
    \caption{\textbf{Ablation Study.} Impact of removing key components from \textsc{SymGraph}. Results are reported as mean $\pm$ standard deviation over 5 random seeds. The consistent performance degradation across all variants validates the design principles of \textsc{SymGraph}.}
    \label{tab:ablation}
    % Use \columnwidth to fit exactly inside one column
    \resizebox{\columnwidth}{!}{%
    \begin{tabular}{@{}lrrrr@{}}
        \toprule
        \textbf{Model Variant} & \textbf{PROTEINS} & \textbf{Mutagenicity} & \textbf{NCI1} & \textbf{BBBP} \\ 
        \midrule
        \multicolumn{5}{l}{\textit{Ablation 1: Topological Role Encoding via Orbits (Orbits)}} \\
        DT w/o Orbits & 70.0 $\pm$ 3.4 & 78.4 $\pm$ 0.9 & 73.0 $\pm$  1.2 & 84.4 $\pm$  1.4 \\
        RF w/o Orbits & 75.3 $\pm$ 2.6 & 81.6 $\pm$ 0.7 & 81.5 $\pm$  1.1 & 86.7 $\pm$  1.3 \\
        \midrule
        \multicolumn{5}{l}{\textit{Ablation 2: Structural Histograms (Counts)}} \\
        DT w/o Counts & 66.7 $\pm$ 4.7& 80.1 $\pm$ 0.9  & 74.2 $\pm$ 1.0 & 84.6 $\pm$ 1.4\\
        RF w/o Counts & 75.2 $\pm$ 1.4  & 83.4 $\pm$ 1.7  & 82.9 $\pm$ 1.7 & 88.5 $\pm$ 1.2 \\      
        \midrule
        \multicolumn{5}{l}{\textit{Ablation 3: Optimization via Evolutionary Search (GA)}} \\
        DT w/o GA     & 69.2 $\pm$ 2.8 & 79.8 $\pm$ 1.1 & 75.3  $\pm$ 1.2 & 86.3 $\pm$ 1.6 \\
        RF w/o GA     & 76.4 $\pm$ 1.7 & 82.8 $\pm$ 1.2 & 82.7 $\pm$ 1.4 & 88.8 $\pm$ 0.6 \\
        \midrule
        \multicolumn{5}{l}{\textit{Full Models}} \\
        \textsc{SymGraph}-DT & 72.9 $\pm$ 4.1 & 81.4 $\pm$ 1.0 & 76.8 $\pm$ 1.6 & 87.1 $\pm$ 1.2 \\ 
        \textsc{SymGraph}-RF & 78.9  $\pm$ 1.8 & 84.8 $\pm$ 1.1 & 84.7 $\pm$ 1.1 & 89.5 $\pm$ 1.7 \\ 
        \bottomrule
    \end{tabular}%
    }
    \vspace{-5pt}
\end{table}

Beyond accuracy gains, \textsc{SymGraph} achieves dramatic speedups over competing self-explainable methods such as \textsc{LogiX-GIN} and \textsc{PiGNN}. By requiring only CPU computation without the need for GPU acceleration, it delivers $10\times$ to $100\times$ speedups across benchmarks. This efficiency stems from our fundamental architectural choice: by replacing expensive gradient-based optimization with a combinatorial Evolutionary Search over discrete structural hashes, we eliminate the computational bottleneck of end-to-end differentiable training that burdens neural SE-GNNs.

\subsection{How do design choices impact performance?}
\label{sec:exp_ablation}

To validate the architectural principles of \textsc{SymGraph}, we conducted a systematic ablation study, as presented in Table~\ref{tab:ablation}. First, replacing the Orbit-Aware Feature Vector $\mathbf{Z}_v$ with a baseline that concatenates features by hop distance ($0, 1, \dots, L$) resulted in a 3.0\% performance decrease for both DT and RF. This confirms that explicitly partitioning nodes via the automorphism group is essential for resolving structural ambiguities. Second, restricting the Predicate Count Vector $\mathbf{v}_G$ to binary indicators caused a sharp 6.2\% decline in accuracy on PROTEINS. This validates that an explicit ``Bag-of-Predicates'' histogram provides superior discriminative power over simple motif detection. Finally, removing the evolutionary search for predicate refinement led to an average 2.0\% decrease in RF accuracy across benchmarks, highlighting the efficacy of this optimization.

% Finally, disabling the evolutionary search for semantic granularity consistently reduced performance, with RF accuracy dropping by approximately 2.0\% on average across all benchmarks. This support its benift of optimization.

% Unlike hop-wise pooling, which suffers from feature mixing by aggregating non-equivalent nodes symmetrically, our approach preserves the distinct topological roles required for granular reasoning.

% This demonstrates that capturing multiplicity is critical for avoiding the destructive interference often found in continuous ``Bag-of-Vectors'' embeddings.

\subsection{What knowledge can we derive using \textsc{SymGraph}?}
\label{sec:exp_discovery}

We evaluate our explanations against the self-explainable \textsc{LogiX-GIN} and the SOTA post-hoc rule-based explainer \textsc{GraphTrail} \cite{GraphTrail}. As illustrated in Figure \ref{fig:baseline_exp}, both baselines exhibit significant limitations on the Mutagenicity and BBBP datasets. \textsc{LogiX-GIN} identifies irrelevant noise motifs, such as generic nitrogen backbones or benzene rings, that lack discriminative power. Meanwhile, \textsc{GraphTrail} consistently generates chemically invalid subgraphs, notably the x149 motif consisting of a structurally impossible cluster of isolated hydrogen atoms, a phenomenon previously noted in \citep{LogicXGNN}.

In contrast, \textsc{SymGraph} recovers expert-validated chemical patterns that align with established domain knowledge (Figure \ref{fig:our_exp}). For instance, our method successfully identifies the 3-membered heterocycle, aromatic nitro, and nitroso groups 
in the Mutagenicity dataset, all of which are documented motifs that participate in chemical reactions 
leading to DNA damage or mutagenesis~\cite{kazius2005derivation}. Similarly, for the BBBP dataset, 
\textsc{SymGraph} extracts functional motifs such as hemiacetals, branched polyols, and carboxylic 
acids~\cite{pajouhesh2005medicinal}, providing a grounded explanation for the model's 
non-permeability predictions.

Furthermore, our explanation framework provides a significant advantage over canonical explanation paradigms, including subgraph-based explanations~\cite{pignn, PGExplainer} and logic rules over such subgraphs, as seen in \textsc{LogiX-GIN} and \textsc{GraphTrail}. These existing methods typically rely on influential subgraphs that lack explicit semantic granularity, a characteristic that is often insufficient for scientific discovery. By enforcing structural constraints on node and edge attributes, our approach offers superior semantic granularity that aligns closely with the SMARTS patterns~\cite{daylight2019smarts} utilized in Structure-Activity Relationship (SAR) analysis. This alignment allows us to directly identify specific functional groups that explain SAR, as shown by the red highlights in Figure~\ref{fig:our_exp}. For instance, our model identifies that the $\text{CO}_2$ moiety in carboxylic acids is a key determinant in reducing blood-brain barrier permeability (BBBP), aligning with evidence that acidic drugs are regulated by specific monocarboxylic acid transport carriers~\cite{kang1990acidic}. Such fine-grained rules bridge the gap between advances in XAI for GNNs and scientific discovery.

Alignment between our rules and SMARTS patterns is discussed in Appendix~\ref{app:smarts_details}. We further validate \textsc{SymGraph}'s explanations on synthetic benchmarks in Appendix~\ref{sec:validation_synthetic}.

\section{Related Work}
\label{sec:related} 

\paragraph{Self-explainable GNNs.}
Self-explainable GNNs (SE-GNNs) aim to produce explanations intrinsically as part of the inference process. One prominent line of work leverages prototype-based approaches \cite{pignn, ProtGNN, gib}, which compare the embeddings of an input graph to a set of learned prototypes using similarity measures to enable inherently interpretable decision-making. An alternative direction leverages kernel-based methods; for instance, \textsc{KerGNN} \cite{kergnn} employs graph kernels and uses kernel activations to explain model predictions in terms of substructure similarity. More recently, \textsc{LogiX-GIN} \cite{LogiXGIN} attempted to incorporate Graded Modal Logic \cite{DBLP:journals/sLogica/Rijke00} into the GIN architecture to generate rule-based explanations, bridging neural representations with symbolic reasoning.

\paragraph{Post-hoc Logic-based Explanations.} This line of work extracts global logical rules to characterize the behavior of pre-trained GNNs \citep{GCneuron, GLGExplainer, GraphTrail}. These methods typically follow a two-stage pipeline: first mapping graph substructures from hidden embeddings into an \emph{intermediate, abstract} concept space, then optimizing logical formulas over these concepts to generate class-discriminative 
rules. While abstract concepts are subsequently grounded in representative 
subgraphs for interpretability, \citet{LogicXGNN} reveals a critical limitation: 
fidelity is evaluated solely in the intermediate concept space, leaving grounding 
quality unverified and often resulting in chemically invalid explanations. In contrast, our predicates are constructed directly from interpretable structural and semantic primitives, ensuring transparency without intermediate abstraction layers.

\section{Conclusion}
\label{sec:conclusion}

In this work, we present \textsc{SymGraph}, the first standalone symbolic 
framework theoretically proven to break the 1-WL expressivity barrier via 
discrete structural hashing and topological role-based aggregation. Our 
extensive empirical evaluations demonstrate that \textsc{SymGraph} consistently 
outperforms state-of-the-art self-explainable models while achieving $10\times$ 
to $100\times$ speedups in training time using only CPU. \textsc{SymGraph} 
generates strict structural constraints on nodes and edges, providing finer 
semantic granularity than existing explanation paradigms. We demonstrate its 
practical utility by recovering expert-validated Structure-Activity Relationships. 
Future work will apply \textsc{SymGraph} to complex biochemistry datasets to 
evaluate its potential in discovering novel Structure–Activity Relationships.

% In this work, we present \textsc{SymGraph}, the first standalone symbolic framework theoretically proven to break the 1-WL expressivity barrier by leveraging discrete structural hashing and topological role-based aggregation. Our extensive empirical evaluations demonstrate that \textsc{SymGraph} consistently outperforms state-of-the-art self-explainable models while achieving $10\times$ to $100\times$ speedups in training time using only CPU. \textsc{SymGraph} generates strict structural constraints on nodes and edges as its explanations, providing finer semantic granularity than existing explanation paradigms. We demonstrate its practical utility by recovering expert-validated Structure-Activity Relationships. Our future work will apply \textsc{SymGraph} to complex frontier biochemistry datasets to evaluate its full potential in discovering novel Structure–Activity Relationships.

\newpage

\section*{Impact Statement}

This work presents a significant step toward making graph neural networks more transparent and reliable for deployment in high-stakes scientific domains, such as pharmaceutical drug discovery. By replacing opaque black-box embeddings with explicit, auditable logical rules, \textsc{SymGraph} facilitates rigorous human oversight and supports the acceleration of knowledge discovery through the identification of novel structural motifs. This transparency is crucial for ensuring that AI-driven scientific insights are grounded in verifiable chemical logic, thereby reducing the risk of biased or spurious correlations in critical decision-making processes. We do not foresee specific negative societal consequences beyond the general dual-use risks inherent to all predictive machine learning technologies, such as those associated with the synthesis of hazardous compounds; however, the inherent interpretability of our framework provides an additional layer of safety by allowing experts to readily identify and vet the reasoning behind a model’s predictions.

% In the unusual situation where you want a paper to appear in the
% references without citing it in the main text, use \nocite
% \nocite{langley00}
\newpage

\bibliography{example_paper}
\bibliographystyle{icml2026}

%%%%%%%%%%%%%%%%%%%%%%%%%%%%%%%%%%%%%%%%%%%%%%%%%%%%%%%%%%%%%%%%%%%%%%%%%%%%%%%
%%%%%%%%%%%%%%%%%%%%%%%%%%%%%%%%%%%%%%%%%%%%%%%%%%%%%%%%%%%%%%%%%%%%%%%%%%%%%%%
% APPENDIX
%%%%%%%%%%%%%%%%%%%%%%%%%%%%%%%%%%%%%%%%%%%%%%%%%%%%%%%%%%%%%%%%%%%%%%%%%%%%%%%
%%%%%%%%%%%%%%%%%%%%%%%%%%%%%%%%%%%%%%%%%%%%%%%%%%%%%%%%%%%%%%%%%%%%%%%%%%%%%%%
\newpage

\appendix
\onecolumn

\section{Dataset Details}
\label{app:dataset}

We evaluate our approach on a diverse set of graph and node classification benchmarks that are widely used in explainable AI research on GNNs and graph learning. Tables~\ref{tab:dataset_main} and~\ref{tab:dataset_supp} summarize the key statistics of these datasets.

\begin{itemize}
    \item \textbf{Molecular and Biological Graphs.} 
    Mutagenicity, NCI1, and PROTEINS are graph classification datasets from the TUDataset collection~\cite{TUDataset}, while BACE, BBBP, and ClinTox are molecular benchmarks from MoleculeNet~\cite{wu2018moleculenetbenchmarkmolecularmachine}. 
    In these datasets, nodes represent atoms or amino acids, and edges correspond to chemical bonds or spatial proximity.
    Mutagenicity contains compounds labeled by their mutagenic effect on a Gram-negative bacterium (label~1 indicates mutagenic).
    NCI1 classifies molecules according to their anticancer activity.
    BBBP labels molecules based on their ability to penetrate the blood--brain barrier.
    PROTEINS consists of protein graphs, where each graph represents a protein structure labeled by its enzyme class.

    \item \textbf{Synthetic Graph Classification.} 
    Ba2Motifs and BAMultiShapes are synthetic benchmarks designed to evaluate the faithfulness of graph-level explanations.
    Ba2Motifs contains Barabási--Albert (BA) graphs augmented with specific motif patterns, where class labels depend on the presence of target motifs.
    BAMultiShapes consists of 1{,}000 BA graphs with randomly attached structural motifs such as houses, grids, and wheels~\cite{GNNExplainer}. Class~0 includes plain BA graphs or graphs containing a single motif, while Class~1 contains graphs enriched with combinations of multiple motifs. Both datasets exhibit clear ground-truth logic rules.

    \item \textbf{Synthetic Node Classification.} 
    To further assess explanation quality on node-level tasks, we use two large single-graph benchmarks \cite{GNNExplainer}.
    BaShapes consists of a Barabási--Albert base graph (300 nodes) with 80 ``house'' motifs attached, together with random perturbations. Nodes are labeled according to their structural role (base graph, or top, middle, and bottom of the house).
    TreeGrid is constructed from a balanced binary tree (height~8) with 80 $3\times3$ grid motifs attached. Nodes are labeled based on their position within the grid (corner, edge, or center) or the underlying tree.

    \item \textbf{Social Graphs.} 
    IMDB-BINARY is a social network dataset from TUDataset~\cite{TUDataset}, where each graph corresponds to a movie; nodes represent actors and edges indicate co-appearances in scenes.
\end{itemize}

\begin{table}[h]
\centering
\caption{Statistics of standard graph classification datasets used in the main paper.}
\label{tab:dataset_main}
\begin{tabular}{lrrrrrr}
\toprule
& \textbf{Ba2Motifs} & \textbf{BAMultiShapes} & \textbf{Mutagenicity} & \textbf{BBBP} & \textbf{NCI1} & \textbf{PROTEINS} \\
\midrule
\#Graphs & 1{,}000 & 1{,}000 & 4{,}337 & 2{,}050 & 4{,}110 & 1{,}113 \\
Avg.\ $|\mathcal{V}|$ & 25.0 & 40.0 & 30.3 & 23.9 & 29.9 & 39.1 \\
Avg.\ $|\mathcal{E}|$ & 51.0 & 87.0 & 30.8 & 51.6 & 32.3 & 145.6 \\
\bottomrule
\end{tabular}
\end{table}

\begin{table}[h]
\centering
\captionsetup{width=0.75\linewidth}
\caption{Statistics of additional graph and node classification datasets used in the supplementary material. Note that TreeGrid and BaShapes are single-graph node classification tasks, so $|\mathcal{V}|$ and $|\mathcal{E}|$ refer to the total size of the single graph.}
\label{tab:dataset_supp}
\begin{tabular}{lrrrrrr}
\toprule
& \textbf{MUTAG} & \textbf{BACE} & \textbf{HIV} & \textbf{IMDB-BINARY} & \textbf{TreeGrid} & \textbf{BaShapes} \\
\midrule
\#Graphs & 188 & 1{,}513 & 1{,}484  & 1{,}000 & 1 & 1 \\
Avg.\ $|\mathcal{V}|$ & 17.9 & 34.1 & 26.1 & 19.8 & 1{,}231 & 700 \\
Avg.\ $|\mathcal{E}|$ & 39.6 & 73.7 & 55.5 & 193.1 & 1{,}550 & 2{,}105 \\
\bottomrule
\end{tabular}
\end{table}

% To assess scalability, we also benchmark our approach and baselines on large-scale, real-world datasets from \cite{karateclub}: Reddit Threads, Twitch Egos, and GitHub Stargazers. Table~\ref{tab:large_datasets_stats} summarizes their statistics.

% \begin{itemize}
%     \item \textbf{Reddit Threads} (Reddit): Labeled as discussion-based or non-discussion-based. The task is to predict whether a thread is discussion-oriented.
%     \item \textbf{Twitch Egos} (Twitch): Ego networks of Twitch users. The task is to predict whether the central gamer plays a single game or multiple games.
%     \item \textbf{GitHub Stargazers} (GitHub): Social networks of developers who starred popular machine learning or web development repositories. The task is to classify whether a social network belongs to a web or machine learning repository.
% \end{itemize}

% \begin{table}[h]
% \centering
% \caption{Statistics of Graph Datasets: Nodes, Density, and Diameter}\
% \label{tab:large_datasets_stats}
% \begin{tabular}{lrrrrrrrr}
% \toprule
% Dataset & Graphs & \multicolumn{2}{c}{Nodes} & \multicolumn{2}{c}{Density} & \multicolumn{2}{c}{Diameter} \\
% \cline{3-4} \cline{5-6} \cline{7-8}
%                  &             & Min & Max & Min & Max & Min & Max \\
% \midrule
% Reddit      & 203,088 & 11  & 197  & 0.021 & 0.382 & 2 & 27 \\
% Twitch      & 127,094 & 14  & 452  & 0.038 & 0.967 & 1 & 12  \\
% GitHub      & 12,725  & 10  & 957 & 0.003 & 0.561 & 2 & 18 \\
% \bottomrule
% \end{tabular}
% \end{table}

\section{Experimental Setup, Training Details, and Baseline Replication}
\label{app:setup}

\subsection{Experimental Setup}
All experiments are conducted on an Ubuntu 22.04 LTS machine with 128 GB RAM, an Intel(R) Core(TM) i7-12700K processor @ 3.60GHz, and an NVIDIA GeForce RTX 5090.

% , using stratified sampling to preserve class balance across splits.

We employ an 8:2 train--test split for our approach. While certain baselines follow their original 8:1:1 configuration, we retain those settings to ensure faithful reproduction. Crucially, this comparison remains fair, as all methods operate under an identical 80\% training budget, ensuring parity in the structural information available for learning. All experiments are repeated with five different random seeds to ensure robustness. The implementation is carried out in PyTorch.

Notably, despite utilizing a mid-range CPU, our framework outperforms baselines optimized for high-end infrastructure. This highlights our model's efficiency, achieving state-of-the-art results through structural expressivity rather than excessive hardware acceleration.

\subsection{Training Details and Hyperparameters for \textsc{SymGraph}}
\label{app:training_details}

As for our approach, \textsc{SymGraph}, we employ the CART algorithm for all decision trees used in the framework. Table~\ref{tab:symgraph_hyperparams} summarizes the specific hyperparameters utilized for each dataset. 
The parameter \texttt{K\_HOPS} defines the radius of the local computational trees ($h_k$). 
We utilize a leaf-wise growth strategy (\texttt{TREE\_MODE} set to \texttt{'LEAVES'}), where the tree expands by splitting the leaf with the highest loss reduction rather than growing level-by-level (depth-wise).
For datasets requiring the feature optimization described in Algorithm~\ref{alg:evolutionary_search} (PROTEINS, Mutagenicity, NCI1, BBBP), we control the search process via \texttt{POPULATION\_SIZE} ($N_{pop}$) and \texttt{GENERATION} ($N_{gen}$).
Finally, \texttt{JUDGE\_MAX\_LEAVES} and \texttt{JUDGE\_CCP\_ALPHA} serve as constraints for the decision trees trained during the evolutionary search (the "Judge" models); these parameters ensure that the fitness evaluation favors candidate genomes that yield interpretable and generalizable rules.
Note that for the synthetic benchmarks (Ba2Motifs and BaMultiShapes), the structural motifs are sufficiently distinct that extensive evolutionary optimization was unnecessary.

\begin{table}[h]
\centering
\caption{Hyperparameter configuration for \textsc{SymGraph} models.}
\label{tab:symgraph_hyperparams}
% \resizebox scales the content to fit \textwidth. The '!' keeps the aspect ratio.
\resizebox{\textwidth}{!}{%
\begin{tabular}{lrrrrrrr}
\toprule
\textbf{Dataset} & \texttt{K\_HOPS} & \texttt{TREE\_MODE} & \texttt{MAX\_LEAVES\_CAP} & \texttt{POPULATION\_SIZE} & \texttt{GENERATION} & \texttt{JUDGE\_MAX\_LEAVES} & \texttt{JUDGE\_CCP\_ALPHA} \\
\midrule
\textbf{Ba2Motifs}     & 1 & \texttt{'LEAVES'} & --- & --- & --- & --- & --- \\
\textbf{BaMultiShapes} & 1 & \texttt{'LEAVES'} & --- & --- & --- & --- & --- \\
\textbf{PROTEINS}      & 1 & \texttt{'LEAVES'} & 16  & 100 & 5   & 48  & 0.001 \\
\textbf{Mutagenicity}  & 2 & \texttt{'LEAVES'} & 16  & 100 & 5   & 48  & 0.001 \\
\textbf{NCI1}          & 2 & \texttt{'LEAVES'} & 16  & 100 & 5   & 48  & 0.001 \\
\textbf{BBBP}          & 2 & \texttt{'LEAVES'} & 16  & 100 & 5   & 32  & 0.001 \\
\bottomrule
\end{tabular}%
}
\end{table}

\paragraph{Atomic Compositionality versus Deep Compression.}
Our framework deliberately restricts the structural receptive field to low hop counts ($K=1, 2$) to leverage the power of atomic symbolic composition. Unlike deep MPNNs that require stacking multiple layers to capture global context, often at the cost of over-smoothing or diluting precise structural signals, our approach treats global graph topology as the logical union of small, high-frequency predicates. By keeping $K$ low, we prevent the vocabulary explosion inherent to higher-order subgraph matching. This ensures our feature space remains dense and learnable rather than becoming sparse and unique to specific graphs. Consequently, the model reconstructs complex global properties (e.g., dense clusters) by aggregating precise, reusable local rules (e.g., the presence of many triangles), effectively bypassing the need for deep, opaque message-passing chains.

\subsection{Implementation of Message-Passing GNN Baselines}

To evaluate the classification performance of \textsc{SymGraph}, we benchmark it against established Message Passing Neural Network (MPNN) baselines, including \textsc{GraphSAGE}~\citep{hamilton2018}, \textsc{GIN}~\citep{GIN}, \textsc{GCN}~\citep{gcn}, and \textsc{GAT}~\citep{GAT}.

All baselines are optimized using AdamW with a learning rate of $10^{-3}$ and weight decay of $10^{-4}$. Models are trained for a maximum of 3000 epochs using a batch size of 128, a hidden dimension of 32, and a dropout rate of 0.2. To ensure efficiency, we employ early stopping with a patience of 200 epochs, monitoring training accuracy for convergence.

\subsection{Replication of Baseline Self-Explainable GNNs}

We benchmark \textsc{SymGraph} against a comprehensive collection of state-of-the-art self-explainable GNNs. To ensure a fair comparison, we utilize the official implementations released by the original authors, adhering to the default or recommended hyperparameters reported in their respective papers. We introduce minimal modifications only where necessary to integrate these models into our unified experimental protocol. Furthermore, we consulted with the original authors to verify that these adaptations remain faithful to the original methods.

% \textsc{PiGNN}, \textsc{LogiX-GIN},  \textsc{GIB}, \textsc{GNAN}, \textsc{KerGNN}

\textbf{\textsc{GNAN}} We encountered notable issues with \textsc{GNAN}~\citep{gnan}. In our experiments, the training loss consistently failed to converge, rendering the training process ineffective. Similar behavior has been implicitly noted in prior work~\citep{LogiXGIN}, where the model exhibited performance akin to a random classifier. Additionally, the training cost of \textsc{GNAN} was prohibitively high; even on modern hardware, each epoch required substantial time, making large-scale hyperparameter tuning impractical.

We communicated these observations to the original authors, who were unable to provide the original configuration used to generate the reported results. However, they confirmed that the released implementation is sensitive to numerical stability and suggested performing a grid search for hyperparameters. Despite dedicating significant resources to tuning and testing various configurations, we were unable to achieve convergence. Consequently, despite our best efforts, we could not identify a stable configuration to faithfully reproduce the reported performance of \textsc{GNAN} within a reasonable experimental budget.

\textbf{\textsc{KerGNN}} For \textsc{KerGNN}~\cite{kergnn}, we observed the shortest training time among all baseline methods. However, as the original paper explicitly recommends performing a hyperparameter grid search, we conducted an extensive search within the ranges suggested by the authors. To ensure fairness, we selected the optimal configuration based on validation performance under the same experimental protocol—a methodology we subsequently confirmed with the first author. While \textsc{KerGNN} is computationally efficient, its performance proved sensitive to hyperparameter selection; our reported results thus reflect the best performance achievable within the recommended search space and a reasonable computational budget.

\textbf{\textsc{LogiX-GIN}, \textsc{GIB} and \textsc{PiGNN}} For the remaining baseline methods, including \textsc{PiGNN}~\cite{pignn}, \textsc{GIB}~\cite{gib}, and \textsc{LogiX-GIN}~\cite{LogiXGIN}, the reproduction process was stable and straightforward. We were able to maintain the original model architectures, as well as the training and evaluation pipelines, with minimal changes.

Since certain datasets in our benchmark were not included in the original codebases, we modified the data-loading components to support them. For datasets without explicitly reported hyperparameters, we initialized settings based on similar datasets from the original papers and subsequently performed a grid search to select appropriate configurations. Overall, these methods were reliably reproduced under our experimental setup.

% Finaly, For all baseline methods, extensive optimization techniques were required to obtain competitive results, including k-fold cross-validation, learning-rate warmup, and careful learning-rate tuning. These procedures were necessary to stabilize training and achieve the reported performance.
% In contrast, our method does not rely on such exhaustive training tricks and remains stable under default settings, highlighting its robustness and ease of deployment.

\subsection{Reproduction of \textsc{LogiX-GIN}'s Rule Explanations}
\label{app:LogiXGIN_BBBP}

For Mutagenicity, we utilize the original explanations (Figure 1) reported in the paper \cite{LogiXGIN}. To obtain explanations for the BBBP dataset, we employed the official \textsc{LogiX-GIN} repository.\footnote{\url{https://github.com/spideralessio/LogiX-GIN}} Following the recommended configuration, we trained a two-layer Backbone and \textsc{LogiX} model, achieving accuracy comparable to reported results. While the original paper suggests a five-layer architecture may yield superior performance, we found the increased depth rendered rule extraction computationally prohibitive (exceeding 10 hours on an RTX 5090).

Furthermore, as the repository lacks automated scripts for rule summarization, and we received no response to our email inquiries regarding the authors' full extraction methodology, we implemented a manual filtering and translation protocol. We employed the provided notebook (\texttt{nbs/LayerWiseRules.ipynb}) to generate Disjunctive Normal Forms (DNFs) representing subgraph features, along with visualizations as shown in Figure~\ref{fig:all_ids}. To facilitate interpretability, we mapped internal node IDs to their atomic identities and manually filtered trivial outputs (e.g., empty/complete graphs or disconnected atoms). The resulting meaningful patterns, such as six-carbon rings and N-C-C sequences, are visualized in Figure~\ref{fig:all_atoms}. We verified the consistency of these rules by examining 100 correctly classified molecules from the BBBP validation set.

\subsection{Comparative Analysis and Reproduction of \textsc{GraphTrail}}
We evaluated \textsc{GraphTrail} \citep{GraphTrail} using the authors' official implementation. Following the original study, we adopted default hyperparameter settings and consulted the authors to verify our experimental setup, ensuring a faithful reproduction on the Mutagenicity and BBBP datasets.

Our evaluation revealed that a significant portion of the generated subgraph explanations were chemically invalid. For example, we observed instances where hydrogen atoms were placed within aromatic rings—a structural impossibility. The authors confirmed that such artifacts can occur and noted that for their published qualitative results, subgraphs were manually redrawn to rectify these errors. This highlights a critical limitation: the existing pipeline lacks the inherent structural constraints necessary to guarantee chemically valid explanations without manual intervention. In contrast, \textsc{SymGraph} addresses this by generating rules that are natively aligned with chemical principles.

% --- Figure 1: ID Images ---
\begin{figure}[htbp]
    \centering
    % Top Row (1-4)
    \includegraphics[width=0.24\linewidth]{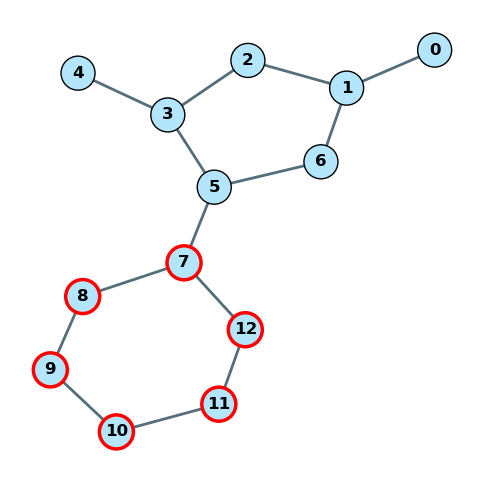}\hfill
    \includegraphics[width=0.24\linewidth]{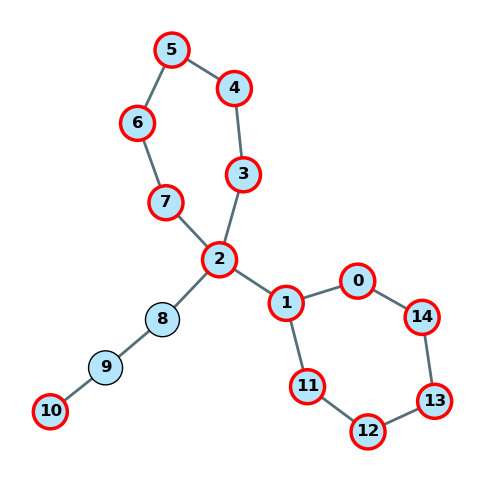}\hfill
    \includegraphics[width=0.24\linewidth]{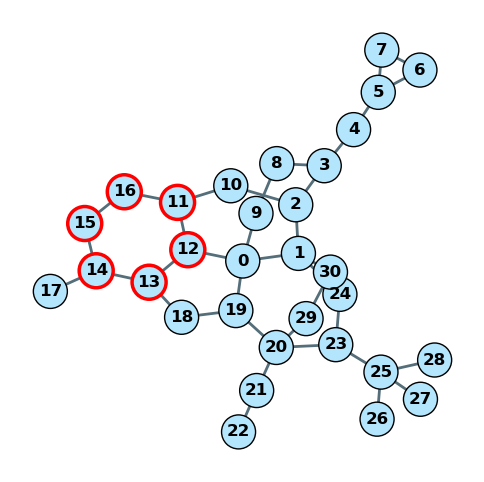}\hfill
    \includegraphics[width=0.24\linewidth]{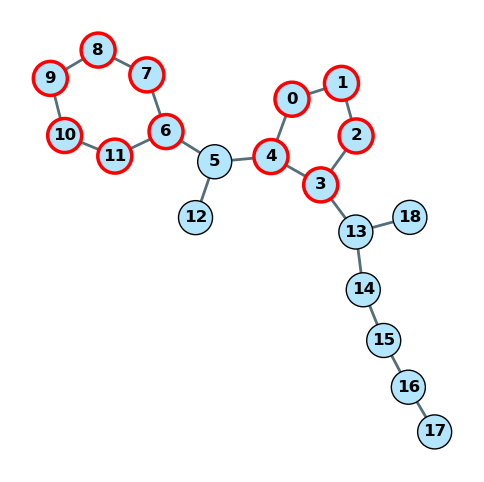}
    
    \vspace{0.5em} % Vertical space between rows
    
    % Bottom Row (5-8)
    \includegraphics[width=0.24\linewidth]{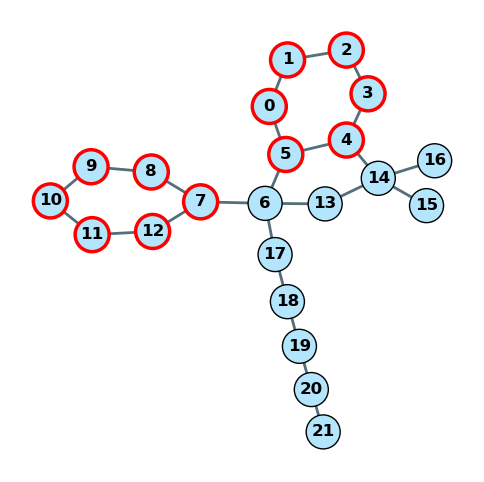}\hfill
    \includegraphics[width=0.24\linewidth]{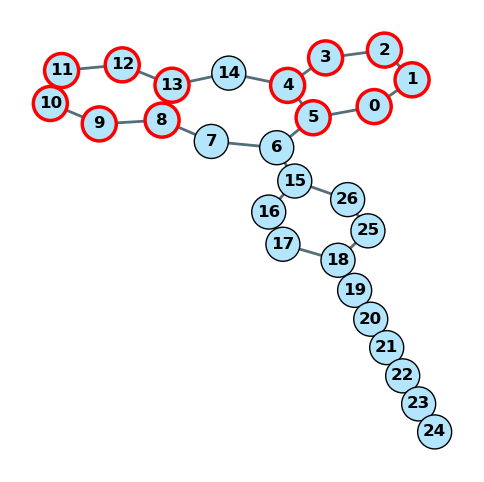}\hfill
    \includegraphics[width=0.24\linewidth]{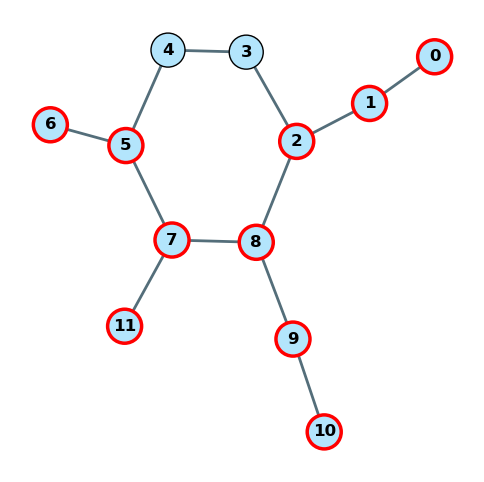}\hfill
    \includegraphics[width=0.24\linewidth]{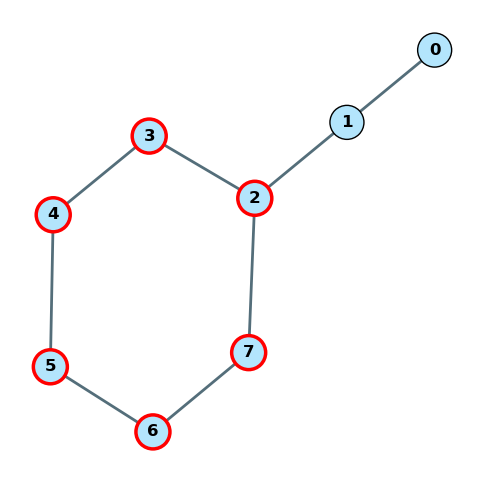}
    
    \caption{\textbf{Raw subgraph visualization via Node IDs.} These figures display partially the raw output from the \textsc{LogiX-GIN} extraction process. The highlighted regions correspond to the internal node indices (IDs) identified in the DNF rules as critical for classification, prior to semantic mapping.}
    \label{fig:all_ids}
\end{figure}

% --- Figure 2: Atom Images ---
\begin{figure}[htbp]
    \centering
    % Top Row (1-4)
    \includegraphics[width=0.24\linewidth]{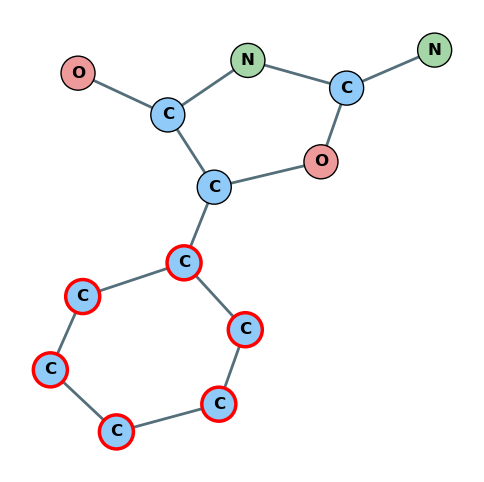}\hfill
    \includegraphics[width=0.24\linewidth]{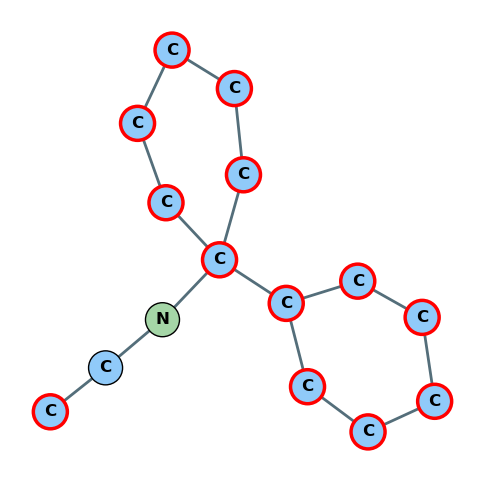}\hfill
    \includegraphics[width=0.24\linewidth]{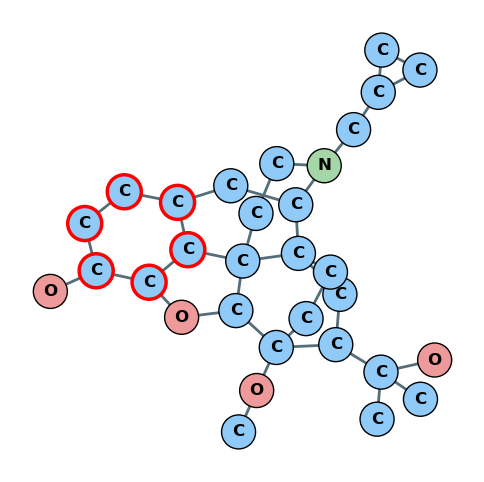}\hfill
    \includegraphics[width=0.24\linewidth]{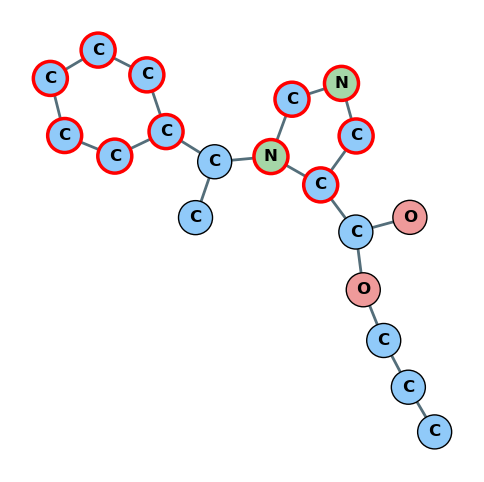}
    
    \vspace{0.5em} % Vertical space between rows
    
    % Bottom Row (5-8)
    \includegraphics[width=0.24\linewidth]{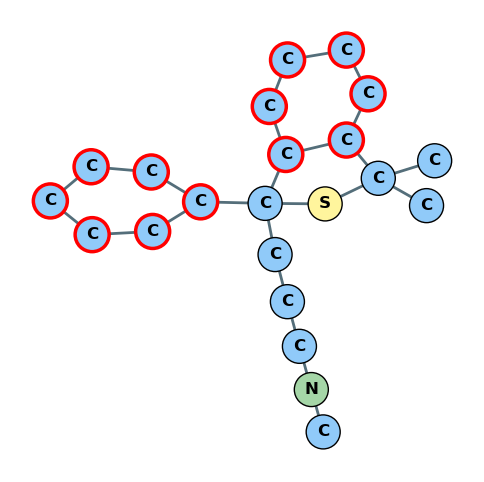}\hfill
    \includegraphics[width=0.24\linewidth]{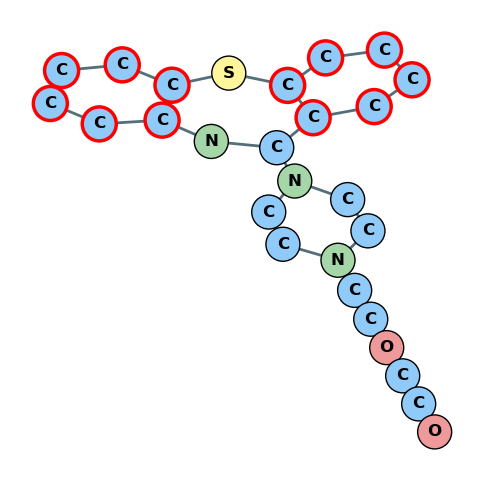}\hfill
    \includegraphics[width=0.24\linewidth]{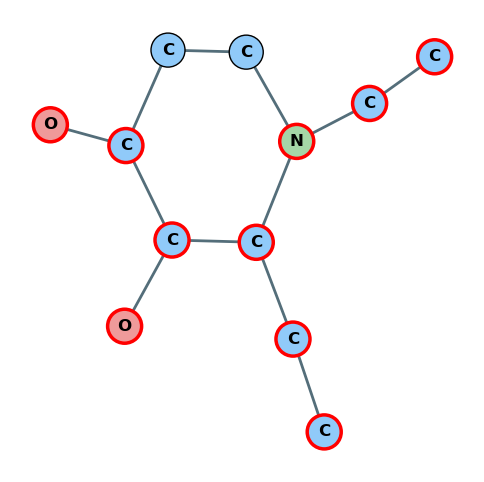}\hfill
    \includegraphics[width=0.24\linewidth]{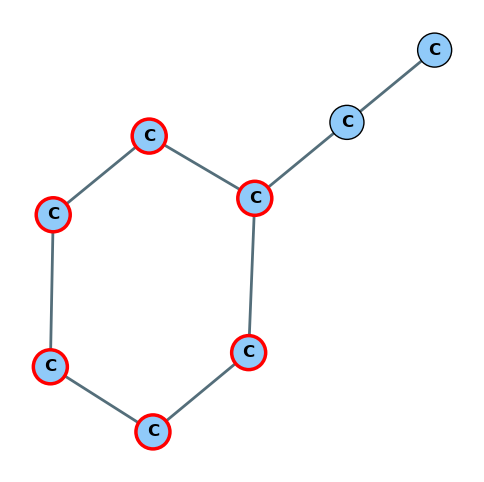}
    
    \caption{\textbf{Semantically interpreted visualization with atom identities.} The corresponding molecular visualization of the graphs in Figure~\ref{fig:all_ids}. By mapping internal IDs to atomic identities and manual filtering, we collect these highlighted subgraphs that are easier to interpret.}
    \label{fig:all_atoms}
\end{figure}

\section{Canonical Representation via Stable Orbit Decomposition}
\label{app:topo_role_encoding_details}

To create a canonical representation $\mathbf{Z}_v$ that is invariant to node indexing and input permutation, we partition the local neighborhood $S_v$ into orbits and establish a consistent ordering using Algorithm~\ref{alg:stable_orbits}. 

Unlike heuristic methods that rely on arbitrary anchor nodes or distances, our approach is purely topological. We utilize the \emph{Weisfeiler-Lehman (WL) structural signature} to strictly define the sort order. Furthermore, we explicitly distinguish between \emph{Node Orbits} (sets of structurally equivalent atoms) and \emph{Edge Orbits} (sets of structurally equivalent bonds), ensuring that interaction information is preserved with high fidelity.

The ordering and construction scheme proceeds as follows:
\begin{enumerate}
    \item \textbf{WL Hashing:} We compute a stable WL-hash $h(u)$ for every node $u \in S_v$. Nodes with identical hashes are grouped into initial Node Orbits.
    \item \textbf{Deterministic Sorting:} These Node Orbits are sorted lexicographically by their hash strings. This guarantees that a specific topological role always appears at a fixed index in the feature vector.
    \item \textbf{Edge Orbit Induction:} Edge Orbits are induced by examining the rank (sort index) of the endpoints. An edge connecting a node of Rank-$i$ to a node of Rank-$j$ forms a deterministic edge role, which is also sorted canonically.
    \item \textbf{Redundancy Filtering:} To maintain a compact representation, nodes that participate in symmetric Edge Orbits are "consumed" and removed from the list of independent Node Orbits. This prioritizes explicit interaction features over isolated node features.
\end{enumerate}

\begin{algorithm2e}[!h]
    \caption{Stable Node-Edge Orbit Decomposition}
    \label{alg:stable_orbits}
    \small
    \DontPrintSemicolon
    \SetKwProg{Fn}{Function}{}{end}
    \KwIn{Graph $S_v$ (local subgraph)}
    \KwOut{Canonical List of Orbits $\mathcal{O}_{final}$}
    \SetKwFunction{SOD}{StableOrbitDecomposition}
    \Fn{\SOD{$S_v$}}{
        
        \tcp{Step 1: Structural Hashing}
        $H \gets \textit{ComputeWLHashes}(S_v)$ \;
        $\mathcal{G}_{node} \gets \text{GroupBy}(S_v.\text{nodes}, \text{key}=H)$ \;
        \tcp{Sort by Hash String (Lexicographical)}
        $\mathcal{O}_{node} \gets \text{Sort}(\mathcal{G}_{node}.values(), \text{key}=H)$ \;
        
        \tcp{Step 2: Edge Orbit Induction}
        $\mathcal{O}_{edge} \gets \emptyset$ \;
        \ForEach{edge $(u,v) \in S_v.\text{edges}$}{
             $Sig \gets \text{Tuple}(\text{Rank}(u), \text{Rank}(v))$ \;
             $\text{AddToGroup}(\mathcal{O}_{edge}, \text{key}=Sig, \text{item}=(u,v))$ \;
        }
        $\mathcal{O}_{edge} \gets \text{Sort}(\mathcal{O}_{edge})$ \;
        
        \tcp{Step 3: Filtering \& Merge}
        $V_{consumed} \gets \{u, v \mid (u,v) \in \cup \mathcal{O}_{edge}\}$ \;
        $\mathcal{O}_{node}' \gets \{O \in \mathcal{O}_{node} \mid O \not\subseteq V_{consumed} \}$ \;
        
        \Return{$\mathcal{O}_{node}' + \mathcal{O}_{edge}$}
    }
\end{algorithm2e}

\begin{theorem}[Orbit Sorting Consistency]
\label{thm:orbit_sorting_consistency}
The WL-based sorting scheme employed in Algorithm~\ref{alg:stable_orbits} produces a deterministic, isomorphism-invariant ordering of orbits. For any two isomorphic subgraphs $S \cong S'$, the sequence of computed hashes and induced edge signatures is identical, ensuring the generated feature vector $\mathbf{Z}(S) = \mathbf{Z}(S')$.
\end{theorem}

\begin{proof}
Let $H(v)$ denote the stable WL-hash of node $v$. For any isomorphism $\sigma: V_S \to V_{S'}$, it holds that $H(v) = H(\sigma(v))$ by the definition of the Weisfeiler-Lehman test. Consequently, the set of unique hash strings present in $S$ and $S'$ is identical. Since $\mathcal{O}_{node}$ is sorted lexicographically by these hash strings, the resulting sequence of Node Orbits is identical for both graphs. Furthermore, an Edge Orbit in $S$ is defined by the pair of ranks $(\text{Rank}(u), \text{Rank}(v))$. Since node ranks are derived from the sorted Node Orbits, they are invariant under isomorphism. Sorting these signatures lexicographically yields an identical sequence of Edge Orbits. Finally, since the filtering step relies solely on set containment within these deterministic orbits, the final output list $\mathcal{O}_{final}$ is invariant.
\end{proof}

\section{Node Classification}
\label{app:node_cls}

For node-level tasks, the local symbolic abstraction provided by the Structure-Aware Predicate $p_j = (h, q)$ is sufficient. Since the semantic state $q$ serves as a discriminative signal learned via direct supervision, the decision process can be formalized as a logical disjunction: $\Phi_y(v) = \bigvee_{p_j \in \mathcal{P}_y} p_j(v)$. 

However, this logical formulation implies exact structural matching, which may limit generalization in the presence of noisy features or structural perturbations. To enhance robustness, we instantiate this logic via the Orbit-Aware Feature Vector $\mathbf{Z}_v$ (Definition \ref{def:orbit_feature}) to create a uniform, robust encoding of a node's neighborhood.

Formally, let $S_v$ denote the local neighborhood of node $v$. We partition $S_v$ into its set of topological orbits and establish a canonical sequence $\mathcal{O}(S_v) = (O_1, \dots, O_B)$ by \emph{sorting the orbits according to their stable Weisfeiler-Lehman (WL) structural signatures}, as detailed in Appendix~\ref{app:topo_role_encoding_details}. This ensures that topologically identical roles (e.g., triangle corners vs. path endpoints) are always assigned to the same fixed position in the feature vector, regardless of their frequency.

To instantiate the generic aggregation in Definition \ref{def:orbit_feature} for node classification, we construct a local descriptor $\mathbf{z}_{O_j}$ for each orbit $O_j$:
\begin{equation}
    \mathbf{z}_{O_j} = \left[ \mathbb{I}(v \in O_j) \oplus |O_j| \oplus \frac{1}{|O_j|} \sum_{u \in O_j} \mathbf{X}_u \right]
\end{equation}
where $\mathbb{I}(\cdot)$ indicates if the center node $v$ belongs to orbit $O_j$, $|O_j|$ represents the structural mass (cardinality), and the final term computes the mean feature aggregation. Note that while $|O_j|$ captures the \emph{prevalence} of a structure, the identity of the structure is guaranteed by the WL-based sorting order.

To map these variable-length descriptors into a fixed feature space compatible with standard classifiers, we employ a structural vocabulary registry $\mathcal{V}$ keyed by the global hash of the subgraph, $h(S_v)$. The final global representation $\mathbf{Z}_v^{\text{global}}$ is constructed as a sparse vector:
\begin{equation}
    \mathbf{Z}_v^{\text{global}} = \text{SparsePlace}\left( \bigoplus_{j=1}^B \mathbf{z}_{O_j}, \;\; \text{idx} = \mathcal{V}[h(S_v)] \right)
\end{equation}
This encoding transforms the node classification problem into a linear decision boundary over topologically distinct orbit configurations, effectively balancing strict structural specificity with feature-based generalization.

\section{Additional Evaluation Results}
\label{app:add_results}

\begin{table*}[th]
\centering
\caption{Test accuracy (\%) and runtime ($10^2$ seconds) on additional benchmarks for XAI graph and node classification. Results are averaged over five random seeds. \textbf{Bold} indicates the best performance, and \underline{underlined} indicates the second-best. \textsc{SymGraph}-DT and \textsc{SymGraph}-RF denote decision tree and random forest variants of our framework, respectively. --- indicates that the method is not supported.}
\label{tab:graph_acc}
\vspace{-5pt}

\resizebox{\textwidth}{!}{%
\begin{tabular}{@{}l rr rr rr rr rr rr@{}}
\toprule
& \multicolumn{2}{c}{\textbf{MUTAG}} 
& \multicolumn{2}{c}{\textbf{BACE}} 
& \multicolumn{2}{c}{\textbf{ClinTox}} 
& \multicolumn{2}{c}{\textbf{IMDB}} 
& \multicolumn{2}{c}{\textbf{TreeGrid}} 
& \multicolumn{2}{c}{\textbf{BaShapes}} \\

\cmidrule(lr){2-3} \cmidrule(lr){4-5} \cmidrule(lr){6-7} \cmidrule(lr){8-9} \cmidrule(lr){10-11} \cmidrule(lr){12-13}

\textbf{Model} 
& Acc$\uparrow$ & Time$\downarrow$ 
& Acc$\uparrow$ & Time$\downarrow$ 
& Acc$\uparrow$ & Time$\downarrow$ 
& Acc$\uparrow$ & Time$\downarrow$ 
& Acc$\uparrow$ & Time$\downarrow$ 
& Acc$\uparrow$ & Time$\downarrow$ \\
\midrule

\textsc{GCN}
& 74.7 $\pm$ 6.1&0.5 $\pm$ 0.2
& 80.7 $\pm$ 1.7 &4.8 $\pm$ 1.1
& 92.3 $\pm$ 0.6 & 8.0 $\pm$ 4.6
& 49.4 $\pm$ 2.1 & 0.3 $\pm$ 0.0
& 71.6 $\pm$ 1.1 & 44.1 $\pm$ 0.8
& 85.5 $\pm$ 2.0 & 34.5 $\pm$ 1.7 \\

\textsc{GIN}
& \underline{88.4 $\pm$ 3.0} & \underline{0.4 $\pm$ 0.0}
&82.4 $\pm$ 2.7 & \underline{3.8 $\pm$ 1.0}
&92.3 $\pm$ 0.7 & \underline{7.8 $\pm$ 2.4}
& 69.5 $\pm$ 1.9 & 0.7 $\pm$ 0.2
&82.6 $\pm$ 0.2 & 34.1 $\pm$ 4.1
& \underline{96.9 $\pm$ 1.0} & 27.6 $\pm$ 1.3 \\

\textsc{GraphSAGE}
& 75.3 $\pm$ 4.8 & 0.4 $\pm$ 0.1
& \underline{82.6 $\pm$ 2.2} & 4.4 $\pm$ 0.8
& 92.2 $\pm$ 1.0 & 8.2 $\pm$ 0.5
& 48.6 $\pm$ 1.5 & \underline{0.2 $\pm$ 0.0}
& 41.5 $\pm$ 2.4 & \underline{24.9 $\pm$ 1.6}
& 42.9 $\pm$ 1.7 & \underline{21.5 $\pm$ 0.2} \\

\textsc{GAT}
& 75.3 $\pm$ 3.5 &0.5  $\pm$ 0.2
&81.7 $\pm$ 1.1 & 6.3 $\pm$ 1.0
&92.6 $\pm$ 0.9 & 9.4 $\pm$ 3.8
& 50.6 $\pm$ 2.1 & 0.3 $\pm$ 0.0
& 41.5 $\pm$ 2.4 & 32.6 $\pm$ 2.0
& 42.9 $\pm$ 1.7 & 28.7 $\pm$ 1.3 \\

% \midrule

% \textsc{GNAN}
% & -- $\pm$ -- & -- $\pm$ --
% & -- $\pm$ -- & -- $\pm$ --
% & -- $\pm$ -- & -- $\pm$ --
% & -- $\pm$ -- & -- $\pm$ --
% & -- $\pm$ -- & -- $\pm$ --
% & -- $\pm$ -- & -- $\pm$ -- \\

\textsc{KerGNN}
& 69.0 $\pm$ 2.2& 19.5 $\pm$ 0.4
& 57.0 $\pm$ 2.7 & 91.8 $\pm$ 0.1

& 92.4 $\pm$ 1.0 & 10.3 $\pm$ 1.0
& 51.3 $\pm$ 1.6 & 65.9 $\pm$ 0.2
& --- & ---
& --- & --- \\

\textsc{GIB}
& 85.3 $\pm$ 8.6 & 21.1 $\pm$ 0.3
& 64.2 $\pm$ 9.0 & 30.5 $\pm$ 0.6
& 92.3 $\pm$ 0.7 & 10.4 $\pm$ 1.2
& 69.2 $\pm$ 8.6 & 21.1 $\pm$ 0.3
& --- & ---
& --- & --- \\

\textsc{PiGNN}
& 77.0 $\pm$ 7.5 & 4.1 $\pm$ 0.1
& 73.4 $\pm$ 4.5 & 49.8 $\pm$ 0.3
& 91.4 $\pm$ 0.6 & 65.4 $\pm$ 3.4
& 64.7 $\pm$ 1.4 & 20.4 $\pm$ 0.7
& 82.4 $\pm$ 6.8 &35.9 $\pm$  1.4
& 85.1 $\pm$ 1.4 &  72.4 $\pm$ 0.7
\\

\textsc{LogiX-GIN}
& 87.4 $\pm$ 2.6 & 6.9 $\pm$ 0.2
& 76.6 $\pm$ 3.1 & 50.7 $\pm$ 0.3
& 92.3 $\pm$ 2.3 & 70.4 $\pm$ 2.7
& 63.6 $\pm$ 4.0 & 12.8 $\pm$ 0.3
& \underline{96.0 $\pm$ 1.0}& 40.7 $\pm$ 0.1
& 94.3 $\pm$ 1.3 & 152.4 $\pm$ 0.1 \\

\midrule

\textsc{SymGraph}-DT
& 82.6 $\pm$  2.3 & \multirow{2}{*}{ \textbf{0.2 $\pm$ 0.0}}
& 80.6  $\pm$ 1.5 & \multirow{2}{*}{\textbf{ 1.1 $\pm$ 0.0 }}
& \textbf{92.9 $\pm$  1.1} & \multirow{2}{*}{\textbf{ 1.1 $\pm$ 0.0}}
& \underline{74.5 $\pm$  2.4}  & \multirow{2}{*}{\textbf{0.3 $\pm$ 0.0}}
& \textbf{100.0 $\pm$ 0.0} & \multirow{2}{*}{\textbf{0.1 $\pm$ 0.0}}
& \textbf{100.0 $\pm$ 0.0} & \multirow{2}{*}{\textbf{0.3 $\pm$ 0.0}} \\

\textsc{SymGraph}-RF
& \textbf{90.5 $\pm$  2.4}  &
& \textbf{83.4 $\pm$ 1.0}  &
& \underline{92.8 $\pm$ 0.9}  &
&   \textbf{78.6 $\pm$ 4.3} &
& \textbf{100.0 $\pm$ 0.0} &
& \textbf{100.0 $\pm$ 0.0} \\

\bottomrule
\end{tabular}
}
\end{table*}

\section{Evolutionary Search Details}
\label{app:ga_algo}

In Section~\ref{sec:training}, we formulated the training of \textsc{SymGraph} as a combinatorial optimization problem over the genome vector $\boldsymbol{\lambda} \in \{1, \dots, \Lambda_{\max}\}^K$. Directly optimizing this vector is computationally prohibitive, as evaluating a single candidate vocabulary configuration traditionally requires retraining the local models for every topological pattern to determine the new semantic states.

To resolve this evaluation bottleneck, we introduce an Evolutionary Search strategy accelerated by \emph{Master Tree Pre-Caching}. For each topological pattern $h_k$, we pre-train a single Master Decision Tree expanded to the maximum allowable granularity $\Lambda_{\max}$. We then cache the resulting leaf index assignments for every node $v$ across all pruning levels into a lookup table matrix $\mathbf{T}_k$. For implementation efficiency, these matrices are organized into a single \emph{Master Lookup Tensor} $\mathbf{T}$ of dimensions $K \times N \times \Lambda_{\max}$. This transformation allows the \textsc{Fitness} function to reconstruct the global feature matrix $\mathbf{V}(\mathcal{D}, \boldsymbol{\lambda})$ via $O(1)$ lookup operations: $q(\lambda_k) = \mathbf{T}[k, v, \lambda_k]$. By bypassing the local retraining phase, we can evaluate thousands of candidate genomes in seconds.

The optimization procedure follows a generational cycle as detailed in Algorithm~\ref{alg:evolutionary_search}. We employ a population $\mathcal{S}$ of genome vectors, using elitism to preserve high-performing configurations. To explore the search space, we utilize tournament selection and uniform crossover. Crucially, we use \emph{creep mutation}, which performs small increments or decrements $(\pm 1)$ to a gene's value. This reflects the ordinal nature of granularity levels, where adjacent levels represent similar semantic resolutions. The objective function maximizes validation accuracy $\Pi$ while applying a $\ell_1$ penalty to control vocabulary inflation:
\begin{equation}
\label{eq:fitness_final}
    \text{Fitness}(\boldsymbol{\lambda}) = \Pi\big( F_{\phi}(\mathbf{V}(\mathcal{D}, \boldsymbol{\lambda})) \big) - \gamma \|\boldsymbol{\lambda}\|_1
\end{equation}

\begin{algorithm2e}[!h]
    \caption{Evolutionary Search with Master Tree Lookup}
    \label{alg:evolutionary_search}
    \small
    \DontPrintSemicolon
    \SetKwProg{Fn}{Function}{}{end}
    \SetKwFunction{Optimize}{EvolutionaryOptimization}
    
    \KwIn{Master Lookup Tensor $\mathbf{T}$, Train Labels $Y_{train}$, Max Granularity $\Lambda_{\max}$, Parameters ($N_{pop}, N_{gen}, \mu, \gamma$)}
    \KwOut{Optimal Genome Vector $\boldsymbol{\lambda}^*$}

    \Fn{\Optimize{$\mathbf{T}, Y_{train}$}}{
        
        \tcp{Initialize population: random granularity levels}
        $\mathcal{S} \gets \textit{InitializePopulation}(N_{pop}, \Lambda_{\max})$ \;
        $\mathbf{Score} \gets \textit{CalculateFitness}(\mathcal{S}, \mathbf{T}, Y_{train}, \gamma)$ \;
        $\boldsymbol{\lambda}^*_{best} \gets \mathcal{S}[\arg\max(\mathbf{Score})]$ \;
        
        \For{$t \gets 1$ \KwTo $N_{gen}$}{
            $\mathcal{S}_{next} \gets \textit{GetTopK}(\mathcal{S}, \mathbf{Score}, k=2)$ \tcc*{Elitism}
            
            \While{$|\mathcal{S}_{next}| < N_{pop}$}{
                $p_1 \gets \textit{TournamentSelect}(\mathcal{S}, \mathbf{Score})$ \;
                
                \eIf{$\text{Random}() < 0.7$}{
                    $p_2 \gets \textit{TournamentSelect}(\mathcal{S}, \mathbf{Score})$ \;
                    $c \gets \textit{UniformCrossover}(p_1, p_2)$ \;
                }{
                    $c \gets p_1$ \;
                }
                
                $c \gets \textit{CreepMutate}(c, \mu)$ \tcc*{Small $\pm 1$ adjustments}
                $\mathcal{S}_{next}.\text{append}(c)$ \;
            }
            
            $\mathcal{S} \gets \mathcal{S}_{next}$ \;
            $\mathbf{Score} \gets \textit{CalculateFitness}(\mathcal{S}, \mathbf{T}, Y_{train}, \gamma)$ \;
            
            \If{$\max(\mathbf{Score}) > \text{Fitness}(\boldsymbol{\lambda}^*_{best})$}{
                $\boldsymbol{\lambda}^*_{best} \gets \mathcal{S}[\arg\max(\mathbf{Score})]$ \;
            }
        }
        \Return{$\boldsymbol{\lambda}^*_{best}$}
    }
\end{algorithm2e}

\section{Proofs on Expressiveness}
\label{app:proofs}

\subsection{Proof of Proposition \ref{prop:local_expressiveness} (Local Expressiveness)}
\label{app:proof_local}

\textbf{Proposition.} \textit{The tuple predicate $p_j = (h, q)$ distinguishes rooted subgraphs that are indistinguishable by standard MPNNs, specifically in cases where: (1) topological structures are distinct but standard MPNNs fail to distinguish them (e.g., due to the 1-WL limitation on regular graphs), or (2) topological structures are identical but feature distributions differ across non-equivalent structural orbits.}

\begin{proof}
We prove this by identifying two distinct classes of graph pairs $(G_1, G_2)$ rooted at nodes $v_1, v_2$ where the standard MPNN update produces identical embeddings $\mathbf{x}_{v_1}^{(L)} = \mathbf{x}_{v_2}^{(L)}$, but our method generates distinct predicate tuples $p_1 \neq p_2$.

\paragraph{Case 1: Structural Ambiguity (Failure on Regular Graphs).}
Standard MPNNs are bounded by the expressiveness of the 1-WL test, meaning they fail to distinguish non-isomorphic graphs that are $k$-regular with the same number of nodes and identical constant features (e.g., a graph consisting of two disjoint triangles vs. a single hexagon).
\begin{enumerate}
    \item \textbf{MPNN Failure:} Let $G_1$ and $G_2$ be two non-isomorphic $k$-regular graphs. The standard MPNN update $\mathbf{x}^{(l+1)}_v = \text{UPD}(\mathbf{x}^{(l)}_v, \sum_{u \in \mathcal{N}(v)} \mathbf{x}^{(l)}_u)$ converges to the same fixed point for all nodes because every node has degree $k$ and receives identical messages from neighbors with identical features. Thus, for any nodes $v_1 \in G_1, v_2 \in G_2$, MPNNs assign $\mathbf{x}_{v_1}^{(L)} = \mathbf{x}_{v_2}^{(L)}$.
    \item \textbf{Our Distinction via $\mathcal{H}(v)$:} Our method computes the discrete topological signature $\mathcal{H}(v) = \text{Hash}(S_v)$. Provided the hashing function acts as a canonical identifier for the induced subgraph $S_v$ (e.g., via canonical string encoding or WL-refinement of the computation graph), and given $S_{v_1} \not\cong S_{v_2}$, the hashes must differ:
    \[ h_1 = \mathcal{H}(v_1) \neq \mathcal{H}(v_2) = h_2 \]
    Consequently, the resulting predicates $p_1 = (h_1, q_1)$ and $p_2 = (h_2, q_2)$ are distinct ($p_1 \neq p_2$), distinguishing the subgraphs where MPNNs fail.
\end{enumerate}

\paragraph{Case 2: Feature Permutation (Orbit-Awareness).}
Consider two rooted subgraphs $S_1, S_2$ rooted at $v_1, v_2$ that are \textbf{topologically isomorphic} (i.e., $\mathcal{H}(v_1) = \mathcal{H}(v_2) = h$) but differ in feature arrangement.
\begin{enumerate}
    \item \textbf{Setup:} Let the neighbors of the root be partitioned into two distinct orbits $O_A$ and $O_B$ (where nodes in $O_A$ cannot be mapped to nodes in $O_B$ by any automorphism). Let $|O_A| = |O_B| = k$. Consider distinct feature values $\alpha \neq \beta$.
    \begin{itemize}
        \item \textbf{In $v_1$ (Graph 1):} Assign feature $\alpha$ to all nodes in orbit $O_A$ and $\beta$ to all nodes in orbit $O_B$.
        \item \textbf{In $v_2$ (Graph 2):} Swap the assignment. Assign $\beta$ to $O_A$ and $\alpha$ to $O_B$.
    \end{itemize}
    \item \textbf{MPNN Failure:} The multiset of neighbor features for both roots is identical: $\{\underbrace{\alpha, \dots, \alpha}_{k}, \underbrace{\beta, \dots, \beta}_{k}\}$. Since standard MPNNs use permutation-invariant aggregation (e.g., sum) over the \emph{entire} neighborhood, the aggregated message is identical:
    \[ \sum_{u \in \mathcal{N}(v_1)} \mathbf{x}^{(0)}_u = k\alpha + k\beta = \sum_{w \in \mathcal{N}(v_2)} \mathbf{x}^{(0)}_w \]
    Thus, standard MPNNs assign identical embeddings $\mathbf{x}_{v_1}^{(L)} = \mathbf{x}_{v_2}^{(L)}$.
    \item \textbf{Our Distinction via $\mathbf{Z}_v$:} Our feature encoding $\mathbf{Z}_v$ concatenates aggregated features according to the canonical orbit ordering $\mathcal{O}(S_v)$.
    \[ \mathbf{Z}_{v_1} = \text{CONCAT}( \dots, \underbrace{\alpha}_{\text{agg}(O_A)}, \dots, \underbrace{\beta}_{\text{agg}(O_B)} ) \]
    \[ \mathbf{Z}_{v_2} = \text{CONCAT}( \dots, \underbrace{\beta}_{\text{agg}(O_A)}, \dots, \underbrace{\alpha}_{\text{agg}(O_B)} ) \]
    Since $\alpha \neq \beta$ and the concatenation order is fixed by the orbit index ($A$ then $B$), the vectors differ: $\mathbf{Z}_{v_1} \neq \mathbf{Z}_{v_2}$.
    Consequently, the learnable function $f_\theta$ maps these distinct vectors to distinct leaf states, yielding $q_1 \neq q_2$. The resulting predicates $p_1 = (h, q_1)$ and $p_2 = (h, q_2)$ are distinct, proving expressiveness strictly greater than 1-WL/MPNNs.
\end{enumerate}
\end{proof}

\subsection{Proof of Proposition \ref{prop:global_expressiveness} (Global Expressiveness)}
\label{app:proof_global}

\textbf{Proposition.} \textit{The Predicate Count Vector $\mathbf{v}_G$ is strictly more expressive than standard global aggregation functions (e.g., Mean, Max, or Min pooling), particularly in its ability to distinguish graphs based on substructure multiplicity.}

\begin{proof}
We prove this by construction. We identify a pair of graphs $G_1$ and $G_2$ that are distinguishable by the Predicate Count Vector $\mathbf{v}_G$ but indistinguishable by standard Mean, Max, or Min pooling GNNs.

Let $\mathcal{T}$ be a specific subgraph (motif) with node feature set $\mathbf{X}_\mathcal{T}$.
\begin{itemize}
    \item Let $G_1$ be a graph consisting of a single instance of $\mathcal{T}$.
    \item Let $G_2$ be a graph consisting of $k$ disjoint copies of $\mathcal{T}$ (where $k > 1$).
\end{itemize}
Assume a standard GNN where the node embedding for any node $u \in \mathcal{T}$ converges to $\mathbf{x}_u$.

\paragraph{Case 1: Standard Global Pooling Failure.}
\begin{itemize}
    \item \textbf{Max Pooling:} The readout function is $\mathbf{x}_G = \max_{v \in V} \mathbf{x}_v$.
    \[
    \mathbf{x}_{G_2} = \max(\underbrace{\{\mathbf{x}_u\}_{u \in \mathcal{T}} \cup \dots \cup \{\mathbf{x}_u\}_{u \in \mathcal{T}}}_{k \text{ times}}) = \max_{u \in \mathcal{T}} \mathbf{x}_u = \mathbf{x}_{G_1}
    \]
    Since $\mathbf{x}_{G_1} = \mathbf{x}_{G_2}$, the classifier cannot distinguish $G_1$ from $G_2$.
    
    \item \textbf{Min Pooling:} Similarly, the readout function is $\mathbf{x}_G = \min_{v \in V} \mathbf{x}_v$. Since the minimum operation is idempotent ($\min(A \cup A) = \min(A)$):
    \[
    \mathbf{x}_{G_2} = \min(\{\mathbf{x}_u\}_{u \in \mathcal{T}} \cup \dots) = \min_{u \in \mathcal{T}} \mathbf{x}_u = \mathbf{x}_{G_1}
    \]
    The graphs remain indistinguishable.

    \item \textbf{Mean Pooling:} The readout function is $\mathbf{x}_G = \frac{1}{|V|} \sum_{v \in V} \mathbf{x}_v$. Let $\mathbf{x}_{\Sigma} = \sum_{u \in \mathcal{T}} \mathbf{x}_u$ be the sum of embeddings in one motif instance. The embedding for $G_2$ is:
    \[
        \mathbf{x}_{G_2} = \frac{1}{k|V_\mathcal{T}|} \sum_{i=1}^k \mathbf{x}_{\Sigma} = \frac{k \cdot \mathbf{x}_{\Sigma}}{k \cdot |V_\mathcal{T}|} = \frac{\mathbf{x}_{\Sigma}}{|V_\mathcal{T}|} = \mathbf{x}_{G_1}
    \]
    Again, $\mathbf{x}_{G_1} = \mathbf{x}_{G_2}$, resulting in indistinguishability.
\end{itemize}

\paragraph{Case 2: Predicate Count Vector Success.}
Let $p_j \in \mathcal{P}$ be a predicate that evaluates to True for a specific node role within the motif $\mathcal{T}$. Let $c_j \ge 1$ be the number of times this predicate is triggered in a single instance of $\mathcal{T}$.

\begin{itemize}
    \item For $G_1$ (single motif), the global count is $C_j(G_1) = c_j$.
    \item For $G_2$ ($k$ disjoint copies), the global count scales linearly: $C_j(G_2) = k \cdot c_j$.
\end{itemize}

Since $k > 1$ and $c_j \ge 1$, it follows that $c_j \neq k \cdot c_j$, and thus $C_j(G_1) \neq C_j(G_2)$. Therefore, a classifier $F_{\phi}$ operating on $\mathbf{v}_G$ (e.g., a simple threshold rule $C_j(G) > c_j$) can trivially distinguish the two graphs.

Thus, the Predicate Count Vector is strictly more expressive regarding substructure multiplicity.
\end{proof}

\section{Examples with Guidance on Reading Logic Rules Induced by Local Structural Predicates}
\label{app:examples_read}

We provide a detailed example of the explanations generated by \textsc{SymGraph} for a single predicate, $p_{1}$, derived from the BBBP dataset. Figure~\ref{fig:ground_example} illustrates the grounding of symbolic rules into a molecular substructure. To enhance interpretability, our framework partitions the local neighborhood into distinct orbits based on topological roles, each labeled and color-coded to signify its structural contribution to the model's decision. This not only improve the experssivenss poweer during training but also allows fine-grained intrepabloty. Unlike canonical subgraph-based explanations that treat a motif as a single entity, the extracted predicate identifies specific node and edge constraints, providing superior semantic granularity.

The induced grounding rules are expressed as a conjunction of conditions over these orbits. To interpret the rule visualized in Figure~\ref{fig:ground_example}, consider the following components:

\begin{itemize}
    \item \textbf{Logic Rule Composition}: The explanation for $p_1$ is defined by the conjunction of conditions on a node orbit and an edge orbit: $(\text{Orbit 0: } \#\text{O} = 1) \wedge (\text{Orbit 1: } \#(\text{C-C}) \geq 2)$.
    
    \item \textbf{Node Orbit (Orbit 0)}: Orbit 0, highlighted in red, identifies a specific central node. The condition \verb|'#O = 1'| mandates that this topological position must be occupied by exactly one Oxygen (O) atom.
    
    \item \textbf{Edge Orbit (Orbit 1)}: Orbit 1 consists of the four teal edges surrounding the center. The condition \verb|'#(C-C) >= 2'| specifies that at least two of these structural bonds must be Carbon--Carbon (C-C) connections.
    
    \item \textbf{Structural Specificity}: A rule is satisfied only when all conjunctive conditions are met simultaneously. In this instance, the model identifies the presence of a specific Oxygen-centered motif supported by a local carbon framework as the discriminative feature for the prediction.
\end{itemize}

\begin{figure*}[!t]
    \centering
    \includegraphics[width=0.95\linewidth]{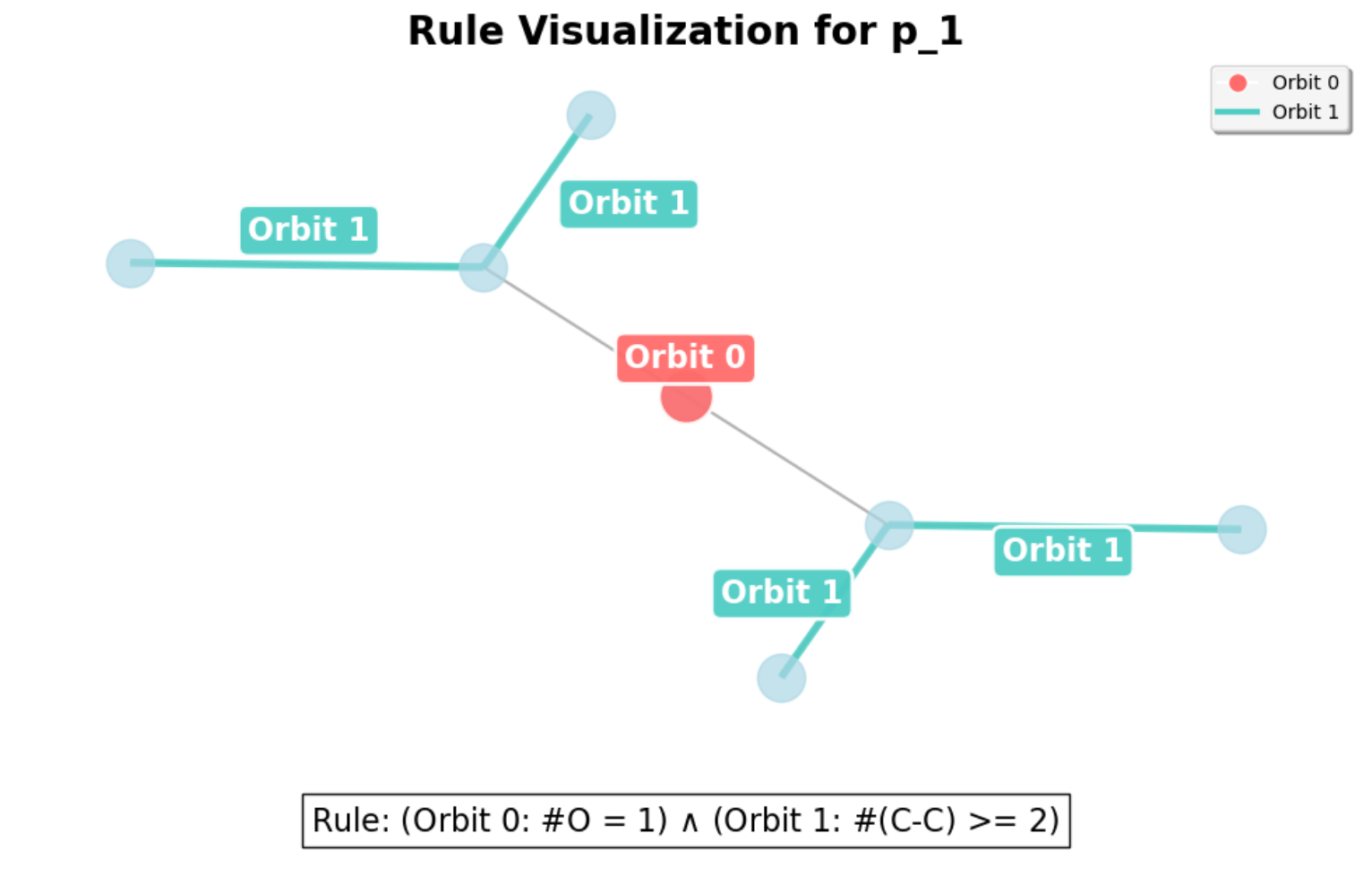}
    \caption{Visualization of logic rules induced by local structural predicates. The symbolic rule $(\text{Orbit 0: } \#\text{O} = 1) \wedge (\text{Orbit 1: } \#(\text{C-C}) \geq 2)$ is grounded to a specific molecular substructure. By decomposing the motif into Orbit 0 (node constraint) and Orbit 1 (edge constraint), \textsc{SymGraph} achieves a level of semantic granularity that exceeds traditional coarse-grained subgraph explanations.}
    \label{fig:ground_example}
\end{figure*}

\section{Scientific Relevance and Alignment with SMARTS Patterns}
\label{app:smarts_details}

While various explanation paradigms offer insights into GNN behavior, they operate at different levels of granularity. This section elaborates on why \textsc{SymGraph}'s fine-grained rule-based explanations are uniquely suited for scientific domains, where interpretability must align with rigorous structural constraints.

In chemistry and pharmacology, experts typically reason via structural formalisms—most notably SMARTS-style substructure patterns \cite{weininger1988smiles, daylight2019smarts}—rather than through isolated examples in Structure–Activity Relationships (SARs). Unlike canonical subgraph-based methods that provide only coarse-grained visualizations, our rules describe general conditions applicable across diverse molecular families and identify the specific functional groups that drive SARs. Table~\ref{tab:smarts_comparison} summarizes the conceptual alignment between the industry-standard SMARTS patterns and our framework.

\paragraph{Implementation and Code Release.} Notably, our framework extracts explicit node and edge constraints directly from local structural predicates. By parsing these constraints via \texttt{rdkit.Chem}, we enable a direct mapping between symbolic rules and standardized chemical representations. To provide seamless utility for the broader scientific community, we are currently expanding our \texttt{RDKit} integration and enhancing system stability. This development is supported by active collaboration with domain experts, whose consistent feedback ensures the interface remains practically relevant for real-world biochemical applications. The full implementation and documentation will be released shortly.

% Notably, our current implementation extracts explicit node and edge constraints directly from local structural predicates. These constraints are parsed using \texttt{rdkit.Chem}, allowing for a direct mapping between symbolic rules and standardized chemical representations. We are currently enhancing the stability of the framework and expanding our integration with the RDKit suite to support a broader range of biochemical applications and provide seamless utility for domain researchers. Furthermore, we are currently collaborating with domain experts and have consistently incorporated their feedback to refine this interface and ensure its practical relevance. The full implementation and documentation will be released to the community shortly.

% While a full-scale discovery of novel chemical properties is beyond the scope of this study, our ongoing work applies \textsc{SymGraph} to complex frontier biochemistry datasets. We aim to examine its capacity to uncover novel structure--activity relationships, positioning our framework as a bridge between methodological GNN advances and practical scientific insight.

\begin{table}[h]
    \centering
    \caption{Conceptual correspondence between chemical SMARTS patterns and \textsc{SymGraph}'s rules.}
    \label{tab:smarts_comparison}
    \renewcommand{\arraystretch}{1.3} 
    \resizebox{\textwidth}{!}{
    \begin{tabular}{@{}l p{0.45\textwidth} p{0.45\textwidth}@{}}
        \toprule
        \textbf{Aspect} & \textbf{SMARTS Rules (Chemical Standard)} & \textbf{\textsc{SymGraph}'s Rules} \\ \midrule
        \textbf{Structural Constraints} & Atom and bond constraints at specific positions (e.g., \texttt{[O]}, \texttt{[C]-[O]}) & Orbit-wise node and edge constraints (e.g., $\#\text{O} = 1$, $\#(\text{C,C}) \ge 2$) \\
        \textbf{Topological Roles} & Defined by bond positions and connectivity & Orbit indices encoding structural symmetry via hashing \\
        \textbf{Logical Flexibility} & Logical \texttt{OR} over pattern variants & Disjunctive Normal Form (DNF) clauses \\
        \textbf{Compositionality} & Recursive conjunctions of conditions & Conjunctions of structural predicates \\
        \textbf{Generalization} & Matches molecules not present in training data & Covers structural families beyond observed motifs \\
        \textbf{Implementation} & Hand-crafted or mined patterns & Automatically induced via RDKit-compatible predicates \\ \bottomrule
    \end{tabular}
    }
\end{table}

\section{Validation on Synthetic Benchmark.}
\label{sec:validation_synthetic}

% \paragraph{Validation on Synthetic Benchmark.}
% In this section, we provide a comprehensive evaluation of the logical explanations generated by \textsc{SymGraph}. We report results for the BA2Motifs, BAMultiShapes, BBBP, and Mutagenicity datasets to demonstrate the interpretability of our method across diverse graph structures. We exclude the NCI1 dataset from this analysis because the mapping between feature indices and atom types is unavailable, preventing the translation of extracted rules into human-readable insights.

We validate our method on the BAMultiShapes dataset, a synthetic benchmark where Class 1 is defined by the ground-truth logic $(H \land W) \lor (H \land G) \lor (W \land G) \Rightarrow \text{Class 1}$, with $H$, $W$, and $G$ denoting House, Wheel, and Grid motifs, respectively. Despite substantial structural noise, our approach accurately recovers these compositional rules (see Figure \ref{fig:BASHAPE}). Due to space constraints, we present a decision tree (DT) of depth 5 and extract its rules for this analysis. By mapping the learned predicates to ground-truth motifs, identifying $\{p_{288}, p_{1104}, p_{1163}\}$ as $H$, $\{p_{576}, p_{824}, p_{712}\}$ as $W$, and $\{p_{594}, p_{1378}\}$ as $G$, we verify that extracted clauses such as $(p_{1378} \land p_{824})$ correspond to $(G \land W)$. Furthermore, our method successfully distinguishes signal motifs from background BA noise (e.g., $p_{378}, p_{1202}$). We note that the full, exact rules can be recovered when the depth is set to 10. This demonstrates robust rule learning capabilities superior to state-of-the-art post-hoc rule-based explainers like \textsc{GraphTrail} \cite{GraphTrail}, which fail to recover comparable logical structures.

While the SOTA self-explainable rule-based GNN, \textsc{LogiXGIN} \cite{LogiXGIN}, claims to identify exact rules, it relies on Graded Modal Logic (GML) \cite{DBLP:journals/sLogica/Rijke00} to aggregate neighborhood degree distributions. Although this enables broader generalization, it risks ``structural hallucination'', i.e., misclassifying trees as houses based solely on degree sequences. Therefore, generating the rules reported in \textsc{LogiXGIN} requires a human-guided induction and summarization phase. Consequently, without known ground truth, the method is often unable to correctly identify the underlying structure due to the inherent limitations of GML.

We similarly validate our approach on the BA2Motifs benchmark, where graphs are distinguished by a House or Cycle motif. As shown in Figure \ref{fig:BAMOTIF}, our method correctly isolates the House motif ($p_{345}$) and recovers the exact minimal logic $(p_{345}) \Rightarrow \text{Class 0}$. This demonstrates robust rule learning capabilities superior to state-of-the-art baselines like GraphTrail \cite{GraphTrail}, which fail to recover comparable logical structures.

\begin{minipage}[t]{0.48\textwidth}
\textbf{BAMultiShapes Class 0 Rules:} (Depth = 5)
\begin{align*}
&(\neg p_{1104} \land \neg p_{1378}) \\
&\lor (\neg p_{1104} \land p_{1378} \land \neg p_{824} \land \neg p_{576} \land \neg p_{288}) \\
&\lor (p_{1104} \land p_{712} \land p_{1378}) \\
&\lor (\neg p_{1104} \land p_{1378} \land p_{824} \land p_{288}) \\
&\lor (\neg p_{1104} \land p_{1378} \land \neg p_{824} \land p_{576} \land p_{288}) \\
&\lor (\neg p_{1104} \land p_{1378} \land p_{824} \land \neg p_{288} \land p_{1163}) \\
&\lor (p_{1104} \land \neg p_{712} \land \neg p_{1202} \land \neg p_{594} \land p_{378}) \\
&\lor (p_{1104} \land \neg p_{712} \land \neg p_{1202} \land p_{594}) \\
&\lor (p_{1104} \land \neg p_{712} \land p_{1202}) \\
&\Rightarrow \text{BAMultiShapes Class 0}
\end{align*}
\end{minipage}
\hfill
\begin{minipage}[t]{0.48\textwidth}
\textbf{BAMultiShapes Class 1 Rules:} (Depth = 5)
\begin{align*}
&(p_{1104} \land \neg p_{712} \land \neg p_{1202} \land \neg p_{594} \land \neg p_{378}) \\
&\lor (\neg p_{1104} \land p_{1378} \land p_{824} \land \neg p_{288} \land \neg p_{1163}) \\
&\lor (\neg p_{1104} \land p_{1378} \land \neg p_{824} \land p_{1163} \land \neg p_{288}) \\
&\Rightarrow \text{BAMultiShapes Class 1}
\end{align*}
\end{minipage}

\begin{minipage}[t]{0.48\textwidth}
\textbf{BA2Motifs Class 0 Rules:} (Depth = 5)
\begin{align*}
&(p_{345}) \Rightarrow \text{BA2Motifs Class 0}
\end{align*}
\end{minipage}
\hfill
\begin{minipage}[t]{0.48\textwidth}
\textbf{BA2Motifs Class 1 Rules:} (Depth = 5)
\begin{align*}
&(\neg p_{345}) \Rightarrow \text{BA2Motifs Class 1}
\end{align*}
\end{minipage}

% \vspace{10pt}

% \textbf{BBBP Dataset Classification Rules} (Depth = 3)
% \begin{flalign*}
% &(\neg p_{38} \land \neg p_{5} \land p_{31}) \lor (\neg p_{38} \land p_{5} \land \neg p_{59}) \lor (\neg p_{38} \land p_{5} \land p_{59}) \lor (p_{38} \land \neg p_{61} \land p_{20}) &&\\
% &\lor (p_{38} \land p_{61} \land \neg p_{54}) \lor (p_{38} \land p_{61} \land p_{54}) \Rightarrow \text{BBBP Class 0}; &&\\[0.5em]
% &(\neg p_{38} \land \neg p_{5} \land \neg p_{31}) \lor (p_{38} \land \neg p_{61} \land \neg p_{20}) \Rightarrow \text{BBBP Class 1} &&
% \end{flalign*}

% \vspace{10pt}

% \textbf{Mutagenicity Dataset Classification Rules} (Depth = 3)
% \begin{flalign*}
% &(\neg p_{2} \land p_{20} \land \neg p_{66}) \lor (p_{2} \land \neg p_{38} \land \neg p_{64}) \lor (p_{2} \land p_{38} \land p_{17}) \Rightarrow \text{Mutagenicity Class 0}; &&\\[0.5em]
% &(\neg p_{2} \land \neg p_{20} \land \neg p_{1}) \lor (\neg p_{2} \land \neg p_{20} \land p_{1}) \lor (\neg p_{2} \land p_{20} \land p_{66}) \lor (p_{2} \land \neg p_{38} \land p_{64}) &&\\
% &\lor (p_{2} \land p_{38} \land \neg p_{17}) \Rightarrow \text{Mutagenicity Class 1} &&
% \end{flalign*}

\begin{figure}[t!]
    \centering
    \includegraphics[width=\textwidth]{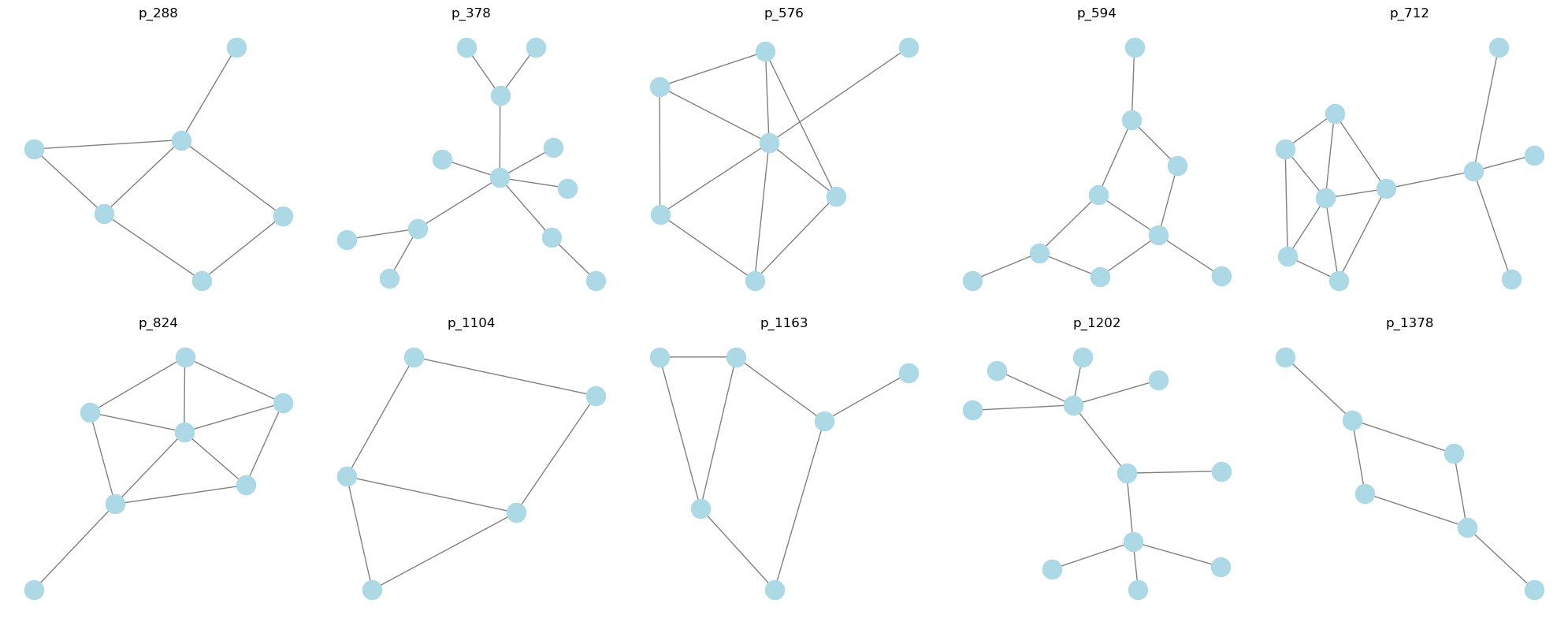}
    \caption{Visualization of the structural predicates identified by \textsc{SymGraph} for the BAMultiShapes dataset.}
    \label{fig:BASHAPE}
\end{figure}

\begin{figure}[t!]
    \centering
    \includegraphics[width=0.4\textwidth]{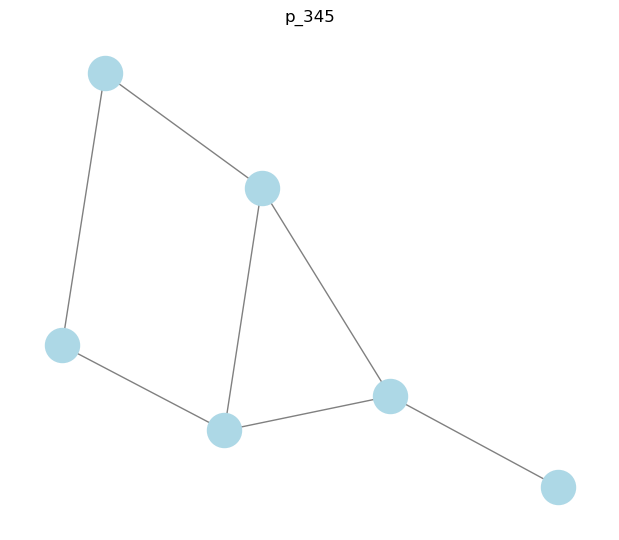}
    \caption{Visualization of the structural predicates identified by \textsc{SymGraph} for the BA2Motifs dataset.}
    \label{fig:BAMOTIF}
\end{figure}

% \begin{figure}[t!]
%     \centering
%     \includegraphics[width=\textwidth]{figures/bbbp_big_0917.pdf}
%     \caption{Our approach's grounded explanation (\textsc{SymGraph}) for BBBP.}
%     \label{fig:BBBP}
%     \vspace{-10pt}
% \end{figure}

% \begin{figure}[t!]
%     \centering
%     \includegraphics[width=\textwidth]{figures/mutag_big_0917.pdf}
%     \caption{Our approach's grounded explanation (\textsc{SymGraph}) for Mutagenicity.}
%     \label{fig:MUTAG}
%     \vspace{-10pt}
% \end{figure}

%%%%%%%%%%%%%%%%%%%%%%%%%%%%%%%%%%%%%%%%%%%%%%%%%%%%%%%%%%%%%%%%%%%%%%%%%%%%%%%
%%%%%%%%%%%%%%%%%%%%%%%%%%%%%%%%%%%%%%%%%%%%%%%%%%%%%%%%%%%%%%%%%%%%%%%%%%%%%%%

\end{document}